\documentclass{article} % For LaTeX2e
\usepackage{iclr2026_conference,times}

\usepackage{amsmath}
\usepackage{amssymb}
\usepackage{mathtools}
\usepackage{amsthm}
\usepackage{tabularx}
\usepackage{subcaption}
\usepackage{booktabs}

\usepackage{algorithm}
\usepackage{algorithmic}
\theoremstyle{plain}
\newtheorem{theorem}{Theorem}[section]

\theoremstyle{definition}
\newtheorem{definition}{Definition}

\theoremstyle{remark}

\usepackage{amsmath,amsfonts,bm}

\def\eqref#1{equation~\ref{#1}}
\def\floor#1{\lfloor #1 \rfloor}
\def\1{\bm{1}}

\DeclareMathAlphabet{\mathsfit}{\encodingdefault}{\sfdefault}{m}{sl}
\SetMathAlphabet{\mathsfit}{bold}{\encodingdefault}{\sfdefault}{bx}{n}

\newcommand{\rev}[1]{\textcolor{black}{#1}}

\DeclareMathOperator*{\argmax}{arg\,max}

\newcommand{\allteams}{\mathcal{T}}
\newcommand{\topkteamsone}{\mathcal{K}_{\allteams(\onevec)}}
\newcommand{\topkteamsgeneric}{\mathcal{K}_{\allteams(\wvec)}}
\newcommand{\alttopkteams}{\mathcal{S}}
\newcommand{\totteams}{M}
\newcommand{\topk}{k}
\newcommand{\amipwvec}{\tilde{\wvec}}

\newcommand{\team}{\widehat{\theta}}
\newcommand{\teamidx}{i}
\newcommand{\altteamidx}{j}

\DeclareMathOperator{\rank}{rank}

\newcommand{\onevec}{1_{N}}

\newcommand{\wvec}{w}
\newcommand{\designmatrix}{\mathbf{X}}

\newcommand{\btheta}{\bm{\theta}}
\newcommand{\bV}{\bm{V}}

\newcommand{\revision}[1]{\textcolor{black}{#1}}

\newcommand{\win}{W}
\newcommand{\loss}{L}
\newcommand{\tie}{T} % JYH ADDED

\usepackage{hyperref}

\usepackage[capitalize,noabbrev]{cleveref} % JYH ADDED

\usepackage{url}
\usepackage[table]{xcolor}
\usepackage{float}  

\theoremstyle{plain}
\newtheorem{proposition}[theorem]{Proposition}
\crefname{proposition}{Proposition}{Propositions}
\Crefname{proposition}{Proposition}{Propositions}

\renewcommand{\algorithmiccomment}[1]{\(\triangleright\) \textbf{#1}}

\title{Dropping Just a Handful of Preferences Can Change Top Large Language Model Rankings}

\author{
Jenny Y. Huang$^{1,2}$\thanks{Equal Contribution},
Yunyi Shen$^{1,2}$\footnotemark[1], 
Dennis Wei$^{2,3}$, 
Tamara Broderick$^{1,2}$ \\
$^{1}$Department of Electrical Engineering and Computer Science, Massachusetts Institute of Technology \\
$^{2}$MIT-IBM Watson AI Lab \qquad
$^{3}$IBM Research \\
\texttt{\{jhuang9,yshen99,tbroderick\}@mit.edu}, \texttt{dwei@us.ibm.com}
}

\iclrfinalcopy % Uncomment for camera-ready version, but NOT for submission.
\begin{document}

\maketitle

\begin{abstract}
We propose a method for evaluating the robustness of widely used LLM ranking systems---variants of a Bradley--Terry model---to dropping a worst-case very small fraction of preference data. Our approach is computationally fast and easy to adopt. When we apply our method to matchups from popular LLM ranking platforms, including Chatbot Arena and derivatives, we find that the rankings of top-performing models can be remarkably sensitive to the removal of a small fraction of preferences; for instance, dropping just $0.003\%$ of human preferences can change the top-ranked model on Chatbot Arena. Our robustness check identifies the specific preferences most responsible for such ranking flips, allowing for inspection of these influential preferences. We observe that the rankings derived from MT-bench preferences are notably more robust than those from Chatbot Arena, likely due to MT-bench's use of expert annotators and carefully constructed prompts. Finally, we find that neither rankings based on crowdsourced human evaluations nor those based on LLM-as-a-judge preferences are systematically more sensitive than the other.

% tl;dr: We present a method for auditing the robustness of LLM ranking systems to worst-case data-dropping; we find that dropping just 0.003% of human preferences can change the top-ranked model on Chatbot Arena.

\end{abstract}

\section{Introduction}
\label{sec:intro}
% It is important to evaluate and align Large Language Models (LLMs) with human value for their safe and effective deployment. Many post-training tasks involve using feedback from humans (e.g., RLHF \citep{chia2023instructeval,stiennon2020learning,bai2022training,christiano2017deep}, Chatbot Arena, \citep{chiang2024chatbot}). 
% \revision{Open evaluation platforms like Chatbot Arena have, in large part due to their openness, become a gold standard for assessing the capabilities of leading LLMs via human preference. These open platforms are now widely used by top LLM developers and companies to evaluate and design new models and benchmarks \citep{chiang2024chatbot,singh2025leaderboard,grattafiori2024llama,hui2024qwen2,white2024livebench}.} % Note: fully understand what chiang2024chatbot meant by widely cited by LLM developers.
\begin{figure}[H]
    \centering
    \begin{minipage}[h]{0.42\textwidth}
        \revision{Open evaluation platforms like Chatbot Arena \citep{chiang2024chatbot} have, in large part due to their openness, become a gold standard for assessing the capabilities of leading LLMs via human preference. These open platforms are now widely used by top LLM developers and companies \citep{wsj2024aiindustry,bloomberg2025platforms} to evaluate and design new models and benchmarks \citep{grattafiori2024llama,hui2024qwen2}.
 Such platforms rely on crowdsourced pairwise battles and human votes to compute model rankings \citep{lee2023rlaif,bai2022training}.}
    \end{minipage}\hfill
    \begin{minipage}[h]{0.56\textwidth}
        \includegraphics[width=\linewidth]{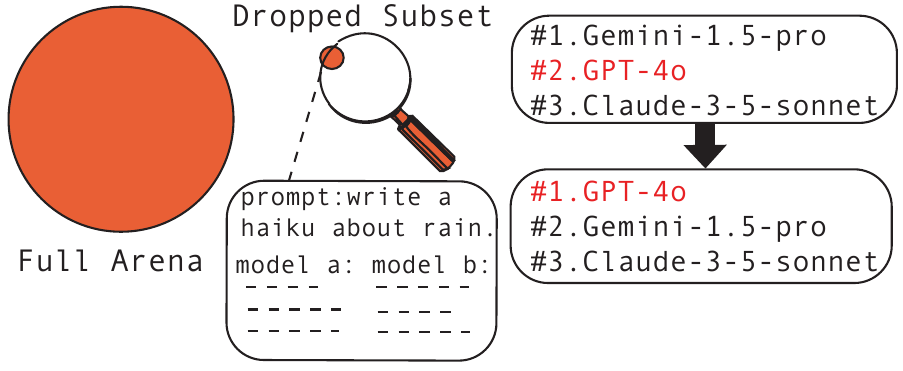}
        \caption{Our method (i) tests whether AI leaderboard rankings remain stable upon dropping small fractions of data and (ii) pinpoints the specific data points (e.g., preferences) that drive ranking flips.}
        \label{fig:paper-overview}
    \end{minipage}
\end{figure}
At the heart of these preference-based evaluation pipelines is the Bradley--Terry (BT) model \citep{bradley1952rank}, which is widely used to rank LLMs based on human feedback \citep{chiang2024chatbot}. The BT model is also used to train reward models for RLHF \citep{ouyang2022training,touvron2023llama,xu2024uncertainty,sun2025rethinking} and route queries to the most appropriate LLM or inference-time scaling strategy \citep{damani2024learning}.

%The global dependence structure of the BT estimation model introduces vulnerabilities that can be exploited by vote-rigging strategies \citep{min2025improving}.
%\jenny{section on vulnerabilities to adversarial agents is now moved to related works. summarize that section briefly in the intro.}
A growing body of work has called into question the trustworthiness of LLM leaderboards, showing that they are vulnerable to adversarial attacks: a few hundred injected votes can change top rankings on Chatbot Arena \citep{min2025improving}, attackers can identify model outputs to systematically upvote or downvote targets \citep{huang2025exploring}, LLM-judges can be easily gamed \citep{zheng2024cheating,raina2024llm}, and issues such as data leakage or selective reporting further undermine leaderboard reliability \citep{singh2025leaderboard}.

In this work, we study a different type of untrustworthiness of LLM ranking systems. That is: \textit{``Will the top rankings \revision{from} LLM-evaluation platforms change upon dropping a very small fraction of the human (or AI) \revision{preference evaluations}?''} A positive answer would raise concerns about the stability and generalizability of rankings produced by such systems. Our notion of non-robustness differs from those of \citet{min2025improving,huang2025exploring,zhao2025challenges} in two major respects. First, it occurs at a different place in the process, at the data analysis step after data has been collected (including from malicious or apathetic users). Second, it does not require adversarial intent. 
%Our notion is more concerned with statistical robustness, namely of a ranking learned from data to dropping a small fraction of the data. While we do aim to find a worst-case fraction, the intent is to provide an upper bound on the degree of non-robustness. %One might worry about whether these rankings, or the learned human preferences from these systems, actually generalize \citep{broderick2023automatic}.

% On the alignment side, a reward model quantifying human preference that cannot generalize might lead to so-called reward hacking \citep{amodei2016concrete,min2025improving}, over optimization \citep{gao2023scaling}, and gaming \citep{pang2022reward}. On the evaluation side, such non-generalizable results can be easy to rig or have surprising behaviors (e.g., the best model turns out to be an LLM that outputs an empty string \citep{zheng2024cheating}).
Our question posed above motivates the need for a systematic way to assess the robustness of top rankings in BT-based evaluation systems to worst-case data dropping. However, a brute-force combinatorial search over all possible small subsets of data would be computationally infeasible for large-scale platforms like Chatbot Arena. %to test whether the model rankings on LLM evaluation systems are robust to the removal of a very small fraction of adversarially-chosen evaluations.
%\footnote{This combinatorial search is computationally infeasible for large-scale platforms like Chatbot Arena.} 
%
%In order to avoid this computationally intractable search,
So we instead turn to a recent line of works from statistics and theoretical computer science that design algorithms for assessing whether data analyses are robust to dropping a small, worst-case fraction of data points \citep{broderick2023automatic,kuschnig2021hidden,moitra2022provably,freund2023towards,shiffman2023could,nguyen2024mcmc,huang2024approximations,rubinstein2024robustness}. One such method, the Approximate Maximum Influence Perturbation (AMIP), estimates how much a statistic of interest could change if a worst-case subset of the data were dropped \citep{broderick2023automatic}. We extend these ideas to develop a fast approximation method for assessing the robustness of \revision{rankings from} LLM evaluation systems to \revision{worst-case} data-dropping. 

We apply our method to assess several popular LLM ranking platforms, including Chatbot Arena and derivatives \citep{chiang2024chatbot,zheng2023judging,miroyan2025search,vichare2025webdev,chou2025visionarena} and find most to be non-robust to dropping a very small fraction of votes. 
%\revision{Although our methods are largely based on the ranking system used in Chatbot Arena and
% MT-bench, there are a few differences (including tie-handling and post-processing) between the raw BT scores we analyze and the
% scores implemented on these platforms; see \cref{sec:deviations-from-reality} for more details.} 

In \cref{sec:setup}, we formalize the setup for assessing worst-case data-dropping robustness in BT-based ranking systems, and in \cref{sec:method} we introduce a computationally efficient method for \revision{assessing} this form of robustness in practice (\cref{fig:paper-overview}). In \cref{sec:experiments}, we apply our robustness assessment method to investigate the robustness of several LLM leaderboards.

%Our code is publicly available at \url{https://github.com/JennyHuang19/IsRankingRobust}, including all scripts required to run our robustness assessing method and to reproduce the results presented in this paper.
% Our method also directly identifies impactflu subsets of evaluations within large alignment data corpuses that led to this change in ranking.
% We find that dropping a surprisingly small handful of impactful human (or LLM) evaluations can change the rankings of top models on these platforms.

% jh: bring up the point on the saturation of current benchmarks.

% jh: specify that the method pinpoints the examples to drop and that the method does not result in false positives. every fraction is true.

\section{Setup}
\label{sec:setup}

\paragraph{Human preference data.} We consider a preference-based ranking system akin to Chatbot Arena \citep{chiang2024chatbot}. There are in total $M$ language models. Any user can submit a prompt to be answered by a pair of language models. Let the $n$th such prompt be sent to models $i_n$ and $j_n$ for $i_n,j_n\in [M] := \{1,\ldots,M\}$ with $i_n\ne j_n$. The user then determines if the response from model $i_n$ is better than that of model $j_n$, or is tied. Suppose there are in total $N$ such comparisons; the $n$th comparison can be seen as a tuple $(i_{n}, j_{n}, y_{n})$, with $y_{n}\in \{\win,\loss, \tie\}$ for whether in the $n$th match, model $i_n$ is preferred over model $j_n$ (a win, $\win$), $j_n$ is preferred over $i_n$ (a loss, $\loss$), or the two models are similar (a tie, $\tie$). From a collection of preference data, the goal is to rank the language models.

\textbf{Ranking with the (unweighted) Bradley--Terry model.} 
\label{sec:bt-setup}
The Bradley--Terry (BT) model is a classical statistical model used to rank players from \textit{binary} match outcomes when there are only wins and losses, $y_n\in \{\win, \loss\}$. In this model, each player (\revision{e.g.}, language model), $i$, is associated with a \textit{BT score}, $\theta_{i}$, and the outcomes are modeled as 
\begin{equation}
    I_{y_{n}=\win}\sim \text{Bernoulli}(\sigma(\theta_{i_n}-\theta_{j_n})),
\end{equation}
where the sigmoid function $\sigma(x) = 1 / (1 + e^{-x})$ and $I$ is the indicator function. 
Note, since the ``winning'' probability depends on the difference between two players' scores rather than on their raw scores, the scores are identified only up to a constant additive term. There are different ways to avoid this identifiability problem \citep{wu2022asymptotic}. Chatbot Arena chooses to set \texttt{mixtral-8x7b-instruct-v0.1} as the reference model, assigning it a fixed score of $1{,}114$.
%We can cast this model as a logistic regression with a specially-structured design matrix. We denote the corresponding ``design'' vector of the $n$th comparison, $x_{n}\in \{-1,0,1\}^{M}$, a vector encoding which two \revision{players} are being compared. That is, if the game is between players $i$ and $j$, then $x_n$ has a $1$ in the $i$th element, a $-1$ in the $j$th element, and $0$ otherwise. Using this structure, we can rewrite the model as a logistic regression model with $M-1$ parameters corresponding to the scores of the players, $\btheta=(\theta_1,  \dots, \theta_{M})\in \mathbb{R}^{M}$ with $\theta_1=0$,
%\begin{equation}
%    y_{n}\sim \text{Bernoulli}(\sigma(x_{n}^\top\btheta )).
%\end{equation}
Chatbot Arena computes the BT-scores (i.e., the estimates of $\btheta=(\theta_1,\dots, \theta_M)$) for the unweighted BT-model by maximum likelihood, 
\begin{equation}
\begin{aligned}
\hat{\btheta} := 
&\argmax_{\btheta:\theta_1=0} \;
\sum_{n=1}^N \big( 
    I_{y_{n}=\win} \log \sigma(\theta_{i_n}-\theta_{j_n})  + I_{y_{n}=\loss} \log (1 - \sigma(\theta_{i_n}-\theta_{j_n}))
\big).
\end{aligned}
\label{eq:BTlikelihood}
\end{equation}

Finally, the \textit{rank} of a model is its position in the sorted list of models, $(\team_{(1)}, \ldots, \team_{(M)})$, ordered by their scores in descending order, so that $\team_{(1)}$ corresponds to the top-scoring and top-ranked model.\footnote{For ease of exposition in the main text, we assume there are no ties. In practice, we did not observe ties in Bradley--Terry-based point estimates. For rankings based on confidence intervals, we do observe and handle ties; see \cref{sec:sensitivity-confints} for details.}
%Finally, we define the \textit{rank} of a model to be the order statistic, $\team_{(i)} \in (\team_{(1)}, ..., \team_{(M)})$.

\textbf{Ranking with the weighted Bradley--Terry model to handle ties.} The classic BT model cannot handle ties. To handle ties, Chatbot Arena adds weights to \cref{eq:BTlikelihood}, counting a tie as both a win and a loss \citep{chiang2024chatbot}.\footnote{``Chatbot Arena Leaderboard Calculation (Bradley--Terry model)'' Colab notebook:\url{https://colab.research.google.com/drive/1KdwokPjirkTmpO_P1WByFNFiqxWQquwH}.} In the weighted BT model, one specifies a weight for wins and losses, $w_{\win\loss}$, and a weight for ties $w_{\tie}$. That is, we estimate BT scores by maximizing the weighted likelihood,
\begin{align}
    \nonumber
    \hat{\btheta} := 
    &\argmax_{\btheta:\theta_1=0} \;
    \sum_{n=1}^N \big[ 
        w_{\win\loss}I_{y_{n}=\win} \log \sigma(\theta_{i_n}-\theta_{j_n})  + w_{\win\loss}I_{y_{n}=\loss} \log (1 - \sigma(\theta_{i_n}-\theta_{j_n}))\\
        \label{eq:weightedBTlikelihood}
        &\qquad\qquad\quad + w_{\tie}I_{y_{n}=\tie}\left\{\log \sigma(\theta_{i_n}-\theta_{j_n})+\log (1 - \sigma(\theta_{i_n}-\theta_{j_n}))\right\}
    \big].
\end{align}
As done on Chatbot Arena, we use $w_{\win\loss}=2$ and $w_{\tie}=1$. This choice can be interpreted as each win or loss counting as two matches of the same outcome, and a tie counting as one win and one loss. They also suggested an alternative treatment of dropping all ties and using the unweighted BT model, which corresponds to $w_{\win\loss}=1$ and $w_{\tie}=0$.

\textbf{Postprocessing in Chatbot Arena.}
\label{sec:postprocessing-chatbot-arena}
Chatbot Arena applies a linear transformation to the learned BT scores \citep{lmarena2024notebook}. They use \(\texttt{SCALE} = 400\),  \(\texttt{INIT\_RATING} = 1{,}000\), and a further shift \(\texttt{ANCHOR\_SHIFT}\) to produce the displayed scores:
\[
\texttt{ELO}_i = \texttt{SCALE} \cdot \widehat{\theta}_i + \texttt{INIT\_RATING} + \texttt{ANCHOR\_SHIFT}.
\]
The final constant (\texttt{ANCHOR\_SHIFT}) shifts all the $\texttt{ELO}_i$ scores so that a specific reference model has a certain score. Chatbot Arena uses \texttt{mixtral-8x7b-instruct-v0.1} as the reference model, assigning it a fixed score of $1{,}114$. We use the same reference model in our analysis of Chatbot Arena; however, we assign the model a fixed score of $0$ (a design choice that does not impact rankings). %and adjusting all other scores accordingly. 
%Adding \texttt{INIT\_RATING} shifts all scores into the thousands range and recentering around ''mixtral-8x7b-instruct-v0.1'' perhaps serves the purpose of creating reference point for comparison across experiments. 
We note that the affine transformation does not affect model rankings since it is strictly monotonic and does not affect our proposed procedure since linear transformations can commute with first-order Taylor expansion.

%We begin by fitting a logistic regression model to pairwise match outcomes using the Bradley--Terry (BT) model. For a set of players, or chatbots, we construct a design matrix \(\designmatrix\), where each row encodes a comparison between two players, and a response vector \(y\), indicating the match result. We train a logistic regression model without regularization and without an intercept. The resulting coefficient vector \(\hat{\theta}\) gives a score for each player (relative to a reference player).

% In the rest of this section, we turn to the goal of determining whether there exists a small fraction of data (e.g., \revision{matchups}) that we can drop to change \revision{the ordering of the estimated BT
% scores.}
%$( \team_{(1)}, ..., \team_{(M)})$.
\textbf{Setup for Data-Dropping.} We study whether dropping a small fraction $\alpha \in (0,1)$ (e.g., $\alpha = 0.01$) of the preference data can change the ordering of the estimated BT scores. \citet{broderick2023automatic} define the \textit{Maximum Influence Perturbation} as the largest possible change induced in a quantity of interest by removing at most $100\alpha\%$ of the data.

Let $w_n$ denote a weight on the $n$th data point, and collect these into a vector $\wvec := (w_1, ..., w_N)$. Define the weighted estimator as
\begin{align}
\nonumber
\hat{\btheta}(\wvec) :=  
    &\argmax_{\btheta:\theta_1=0} \;
    \sum_{n=1}^N w_{n}\big[ 
        w_{\win\loss}I_{y_{n}=\win} \log \sigma(\theta_{i_n}-\theta_{j_n}) + w_{\win\loss}I_{y_{n}=\loss} \log (1 - \sigma(\theta_{i_n}-\theta_{j_n}))\\
        &\quad\quad\qquad\qquad\quad + w_{\tie}I_{y_{n}=\tie}\left\{\log \sigma(\theta_{i_n}-\theta_{j_n})+\log (1 - \sigma(\theta_{i_n}-\theta_{j_n})\big)\right\}
    \big].
\end{align}
Setting $\wvec = \onevec$ (the all-ones vector) recovers the BT scores computed on the full data (e.g., the original arena), while setting $w_n = 0$ corresponds to dropping the $n$th data point (e.g., \revision{a matchup}). We define the set of all weight vectors corresponding to dropping at most an $\alpha$-fraction of the data as follows.

\begin{definition}[Feasible Drop Set]
\label{def:drop-set}
Let $W_\alpha := \{ \wvec \in \{0,1\}^N : \sum_{n=1}^N (1 - w_n) \leq \alpha N \}$ be the set of all binary weight vectors indicating subsets where at most $100\alpha\%$ of the data has been dropped.
\end{definition}

%\textbf{Two-Player Arena.} 
We begin by considering \rev{the ordering of BT scores between a pair of players, $i$ and $j$.}
%an arena involving just two players (e.g., LLMs): player $i$ and player $j$. 
Without loss of generality, we assume\footnote{If this assumption does not hold, the identities of $i$ and $j$ can be swapped.} that player $i$ has the higher estimated BT score on the full data:
\[
\team_i(\onevec) \geq \team_j(\onevec).
\]
We are interested in whether this ordering can be reversed by dropping at most an $\alpha$-fraction of the data.

% \begin{definition}[Top-1 Data-Dropping Robustness in Two-Player Arenas]
% We say that an arena consisting of players $i$ and $j$ is \textit{top-1 robust at level $\alpha$} if there does \emph{not} exist a data weighting $\wvec \in W_\alpha$ such that the BT scores reverse under reweighting:
% \begin{align}
% \left\{ \wvec \in W_\alpha : \team_i(\wvec) < \team_j(\wvec) \right\} = \emptyset.
% \label{eqn:robust-two-model-arena}
% \end{align}
% \end{definition}

% To generalize this setup beyond a two-player arena, we introduce more notation.

%\textbf{$M$-Player Arena.} 
We now extend this notion to \rev{an} arena with $M$ players, for any $M \geq 2$. Let $\allteams(\wvec) := \{ \team_{\teamidx}(\wvec) \}_{\teamidx=1}^{\totteams}$ denote the set of BT scores under weighting $\wvec$. Let $\rank\left[\team_{\teamidx}(\wvec); \allteams(\wvec)\right]$ denote the rank of a model under the weighting $\wvec$.

\begin{definition}[Top-$k$ Set]
We define the \textit{top-$\topk$ set} under a data weighting $\wvec$ as the set of players whose scores rank among the top $\topk$:
\begin{align}
\topkteamsgeneric := \left\{ \teamidx :  \rank\left[\team_{\teamidx}(\wvec); \allteams(\wvec)\right] \le \topk \right\}.
\label{eqn:top-k-definition}
\end{align}
\end{definition}

\begin{definition}[Top-$k$ Data-Dropping Robustness] %in $M$-Player Arenas
An arena is \textit{top-$\topk$ robust at level $\alpha$} if no $\alpha$-fraction subset of data can be dropped to change the top-$\topk$ set for the full data. That is,
\begin{align}
\left\{ \wvec \in W_\alpha : \topkteamsone \neq \topkteamsgeneric \right\} = \emptyset.
\label{eqn:robust-k-model-arena}
\end{align}
\label{defn:top-k-robustness}
\end{definition}
Notice that 
%both \cref{eqn:robust-two-model-arena} and 
\Cref{eqn:robust-k-model-arena} \rev{is} nontrivial to directly verify; to check directly, we could test out dropping all possible small-fraction subsets of the arena, but this combinatorial operation is computationally intractable in practice.
%, namely (1) rerun the analysis and rank the newly estimated parameters upon dropping and (2) test out dropping all possible small-fraction subsets of the arena. 

In \cref{sec:method}, we show that verifying whether \cref{eqn:robust-k-model-arena} holds can be reduced to checking the robustness of a series of pairwise comparisons. Specifically, top-$\topk$ robustness as in \rev{\cref{defn:top-k-robustness}} can be checked by assessing whether there exists a reweighting $\wvec \in W_\alpha$ that flips the ranking of a pair $(i, j)$ such that $i$ is inside and $j$ is outside the top-$\topk$ set. We then can test if such flipping can happen by: (a) using a continuous approximation of the discrete weights $w$ (also known as ``approximate data-dropping") to identify a promising candidate subset of influential preferences, (b) dropping these, (c) recomputing the BT-based rankings, and (d) observing whether the rankings change. We detail this procedure in \cref{sec:method}.

\section{Proposed method}
\label{sec:method}
Recall that our goal is to evaluate the robustness of the rankings induced by a BT-model when a small fraction of matches (e.g., evaluations) is removed from the arena. To this end, we introduce a method based on checking the robustness of pairwise BT score differences. 
\rev{We provide pseudocode for our method in \cref{alg:isRankingRobust} and explain its steps below.}
%With this pairwise formulation, we can use key ideas from a method that allows for the quick identification of data subsets whose removal can change the sign of statistics-of-interest (e.g., a particular effect size in a logistic regression). 
%In this work, we use AMIP to check for sign changes in pairwise BT score differences between players (i.e., a difference between two given effect sizes in a logistic regression).

In \rev{\cref{prop:topk}}, we show that a top-$\topk$ set can be characterized by considering a set of pairwise comparisons. \rev{This result allows us to check top-$\topk$ robustness by checking pairwise robustness of all models inside the top-$\topk$ set against all models outside of this set. In the case that there does exist such a pair of models (one inside and one outside the top-$\topk$) whose rankings flip, then the top-$\topk$ set has changed, i.e., the arena is non-robust. In the case that there does not exist at least one such pair of models whose rankings can be flipped upon dropping a small fraction of preferences, then the top-$\topk$ set remains unchanged, i.e., the arena is top-$\topk$ robust.}

\rev{Given the equivalence between checking the robustness of the top-$\topk$ set and checking the robustness of the aforementioned series of pairwise player comparisons,} we propose a greedy algorithm to test whether the top-$\topk$ set is robust to worst-case data-dropping. 
Namely, we test the data-dropping robustness of all players in the top-$\topk$ set against all players outside of the top-$\topk$ set. 
%If any one of these pairwise comparisons is non-robust, then the top-$\topk$ set is non-robust, since one of the members of the top-$\topk$ set will have been exchanged for an element outside the top-$\topk$.

Before that, we describe what it means for a given pair of player scores, $(\team_i(\wvec), \team_j(\wvec))$, to be data-dropping robust. Without loss of generality, we assume throughout this section that player $i$ has the higher estimated BT score on the full data.

\textbf{Pairwise Robust.} Given a pair of \rev{players}, $(i,j)$, we say that the scores for this pair, $(\team_i(\wvec), \team_j(\wvec))$, are robust to small-fraction data-dropping at level-$\alpha$ if 
\begin{equation}
    \{ \wvec \in W_\alpha : \team_i(\wvec) < \team_j(\wvec) \} = \emptyset.
    \label{eqn:pair-is-robust}
\end{equation}

\textbf{Top-$\topk$ Robust.}
\label{sec:top-k-robustness-via-pairwise-comparisons}
Recall that an arena is top-$k$ robust at level-$\alpha$ if there does not exist a reweighting, $\wvec \in W_\alpha$, such that $\topkteamsone \neq \topkteamsgeneric$. %(see \cref{eqn:robust-k-model-arena}). 
Using the line of logic in \cref{prop:topk}, this is equivalent to showing that, $\forall \; (i,j)$ where  $i \in \topkteamsgeneric$ and $j \notin \topkteamsgeneric$, the pair $(\team_i(\wvec), \team_j(\wvec))$ is robust. Namely, if every comparison $(i, j)$ in this set of pairwise comparisons stays the same (after reweighting), then the top-$\topk$ set also stays the same (see \cref{prop:topk} for a detailed proof).

We now provide a method for checking the robustness of pairwise comparisons.%\footnote{The code implementation of our method can be found at \href{https://github.com/JennyHuang19/IsRankingRobust}{https://github.com/JennyHuang19/IsRankingRobust}.} 

\textbf{Method for Checking Pairwise Robustness.} In \cref{eqn:pair-is-robust}, we are interested in checking whether there exists a small fraction of evaluations, $\wvec \in W_\alpha$, that can be dropped to change the sign of a difference in BT scores. Without loss of generality, we will assume that the sign of the difference of BT scores fit to the full data is positive (e.g., $[\team_i(\onevec) - \team_j(\onevec)] > 0$, meaning that model $i$ has a higher score than model $j$).

To evaluate the robustness of the sign of $[\team_i(\onevec) - \team_j(\onevec)]$ to dropping a small fraction of matches, we adopt a recently-developed method from the statistics literature known as the \textit{Approximate Maximum Influence Perturbation} \citep{broderick2023automatic} (see \cref{sec:amip-approximation} for a more detailed discussion on how we adapt this method to our problem setup). This method approximates the maximal directional change in a statistic, e.g., $[\team_i(\onevec) - \team_j(\onevec)]$, that can result from dropping a worst-case subset of data points (in our case, evaluations) of size at most $\floor{\alpha N}$. This method allows us to sidestep running an expensive combinatorial search over all data subsets for the worst-case subset of matches to drop, a procedure that is computationally prohibitive for large LLM evaluation platforms like Chatbot Arena. 
%As the name suggests, AMIP approximates the \textit{Maximum Influence Perturbation}, or the maximal change to a quantity-of-interest (e.g., $[\team_i(\onevec) - \team_j(\onevec)]$) that can be induced by dropping a small fraction of data points from a dataset (i.e., by performing the data analysis using a reweighting of the data in which the weight of a small subset of points is set to $0$, $w \in W_\alpha$).

The optimization problem implied by the Maximum Influence Perturbation problem in our particular case is shown below,
\begin{align}
\max_{\wvec \in W_{\alpha}} \left( \left[\team_i(\onevec) - \team_j(\onevec)\right] - \left[\team_i(\wvec) - \team_j(\wvec)\right] \right).
\label{eqn:comb-optimization}
\end{align}
We approximate this discrete optimization problem using AMIP approximation \citep{broderick2023automatic}; the idea is that, instead of solving the optimization directly, we first approximate the effect of dropping data by a first order Taylor expansion of the quantity $\team_i(\wvec) - \team_j(\wvec)$ over data weights $\wvec$ and then solve the approximated optimization problem. In \cref{sec:BTlogistic}, we provide a review of the general AMIP approximation, then formulate both the weighted and unweighted BT models as logistic regressions, and finally provide an explicit form of the approximation for both BT models.

%Let $e_j$ denote the $j$th standard basis vector and $\designmatrix \in \mathbb{R}^{N \times M}$ denote the design matrix. Let \(\widehat{p}_n = \sigma( \hat{\theta}^\top x_n)\) and $\bV = \operatorname{diag}(\{\widehat{p}_n (1 - \widehat{p}_n)\}_n)$. For logistic regression with an effect-size quantity of interest, $\theta_j$, the formula for the influence score for the $n$th data point \citep{pregibon1981logdiagnostics} is given by}
%\revision{\begin{align}
%    \frac{\partial \hat{\theta}_j(\wvec)}{\partial w_n}\Big\vert_{\wvec = 1_N} = 
%    e_j^\top(\designmatrix^\top \bV \designmatrix)^{-1}x_n\widehat{p}_n (1 - \widehat{p}_n)\left(y_n - \widehat{p}_n\right).
%    \label{eqn:logreg-influence-function}
%\end{align}}%

Let the approximate solution to \cref{eqn:comb-optimization} that is returned by
AMIP be denoted as $\amipwvec$ (i.e., the set of data weights that are 0 at indices of data points that AMIP chooses to drop and 1 elsewhere). For a candidate pair of players, \((i, j)\), \rev{w}e check whether after dropping, $ [\team_i(\amipwvec) - \team_j(\amipwvec)] < 0$. In other words, we refit the BT-model upon leaving out the subset of impactful evaluations identified by AMIP and check whether leaving out this subset induces a sign change in the difference of BT scores for the pair, $(i, j)$. We say that the BT scores for a pair of players, \((i, j)\), are non-robust if the \textit{sign} of the difference in scores \textit{becomes negative} upon refitting under $\amipwvec$, (i.e., if $ [\team_i(\amipwvec) - \team_j(\amipwvec)] < 0$).

\textbf{Method for Checking Top-$\topk$ Robustness.} We now describe how we can fold our check for pairwise robustness into an overall routine for checking for top-$\topk$ robustness. 

Recall from earlier in \cref{sec:top-k-robustness-via-pairwise-comparisons} that we can check top-$\topk$ robustness by checking pairwise robustness for every comparison $(i,j)$ where  $i \in \topkteamsgeneric$ and $j \notin \topkteamsgeneric$. This amounts to checking the pairwise robustness for at most \(k(M-k)\) pairs.

Thus, we check top-$\topk$ robustness by iterating over pairs of players. Note that, when checking the robustness of a given pair $(i,j)$, we allow matches between any two models (not only $(i,j)$) to be dropped. Since we only need to find one non-robust pair to render the set non-robust, not all pairs need to be checked. To save on compute, we take a greedy approach and start with comparing the most closely-ranked pairs between the top-\(k\) ranked players and the remaining \(M-k\) players, where ``closeness'' is quantified using the absolute difference in BT scores fit on the full data;\footnote{\revision{The robustness of the relative ranking of two players is correlated with the proximity of their BT scores as seen in \cref{fig:robustness-vs-score-gap}.}} pairs with smaller BT-score gaps are more likely to exhibit data-dropping non-robustness. Upon finding any single pair that is pairwise non-robust at an $\alpha$-level, the procedure terminates early and returns the corresponding players and the indices of the dropped evaluations. We say that an arena is $\alpha$-level top-$\topk$ robust if there does not exist a pair of players $(i,j)$, where $i \in \topkteamsgeneric$ and $j \notin \topkteamsgeneric$, that are $\alpha$-level pairwise non-robust. While our method uses an approximation to \textit{identify} the influential preferences; it then performs an exact recomputation of the Bradley--Terry scores with the identified preferences removed. As a result, all \rev{non-robustness reported in this paper is} definitive: when we state that dropping $100\alpha\%$ of preferences changes the ranking, we have explicitly verified that the ranking does in fact change upon removal of the surfaced subset. \rev{However, the algorithm may not catch all cases of non-robustness (i.e., false negatives are possible). See \cref{sec:false-negatives} for an extended discussion on the possibility of false negatives.}

\begin{algorithm}[t]
\small %
\caption{\rev{Our Data-dropping Robustness Check on Rankings}}
\label{alg:isRankingRobust}
\begin{algorithmic}[1]

\STATE \rev{\textbf{Input:} Dataset $\mathcal{D}$ (the collection of matches, e.g., preferences), rank $k$, drop fraction $\alpha$.}
\STATE \rev{\textbf{Output:} (1) A determination of whether top-$k$ non-robustness was found. (2) If top-$k$ non-robustness is found, we additionally return the pair of players ($i,j$) whose rankings flipped, the differences in their scores pre- and post- data-dropping, and the most influential set $I$.}

\STATE

\STATE \rev{\(\triangleright\) \textbf{Fit a Bradley--Terry model on the full arena.}}
\STATE \rev{$\hat{\btheta}(\onevec) \gets \text{FitBTModel}(\mathcal{D})$}
\STATE \rev{Determine the top-$k$ set, $\topkteamsone$, from $\hat{\btheta}(\onevec)$.}

\STATE

\STATE \rev{\(\triangleright\) \textbf{Compute score gap for each player pair of interest.}}
\STATE \rev{$P \gets \{(i,j): i \in \topkteamsone,\, j \notin \topkteamsone\}$}
\FOR{\rev{each $(i,j) \in P$}}
    \STATE \rev{Compute score gap $\hat{\Delta}(\onevec)_{ij} \gets |\team_i(\onevec) - \team_j(\onevec)|$}
\ENDFOR
\STATE \rev{Sort pairs $(i,j)$ in $P$ by increasing $\Delta(\onevec)_{ij}$}

\STATE

\STATE \rev{\(\triangleright\) \textbf{Check pairwise robustness by choosing pairs in order of increasing score gap.}}
\FOR{\rev{each $(i,j)$ in sorted $P$}}

    \STATE \rev{\(\triangleright\) \textbf{Compute influence scores.}}
    
    \FOR{\rev{each datapoint (preference) $n$}}
        \STATE \rev{$\text{IF}_n(i) \gets$ influence score for datapoint $n$ on $\team_i(\onevec)$}
        \STATE \rev{$\text{IF}_n(j) \gets$ influence score for datapoint $n$ on $\team_j(\onevec)$}
        \STATE \rev{$\Delta_n(i, j) \gets \text{IF}_n(i) - \text{IF}_n(j)$}
    \ENDFOR

    \STATE

    \STATE \rev{\(\triangleright\) \textbf{Identify worst-case subset by sorting influence scores.}}
    \STATE \rev{Choose the $\floor{\alpha N}$ values of $\Delta_n(i, j)$ that are the largest in the negative direction, assuming that $\team_i(\onevec) - \team_j(\onevec) > 0$.}
    \STATE \rev{$I \gets$ indices corresponding to the $\floor{\alpha N}$ most negative $\Delta_n$ values.}

    \STATE

    \STATE \rev{\(\triangleright\) \textbf{Compute the AMIP-predicted score difference.}}
    \STATE \rev{$(\hat{\theta}_i(\wvec) -  \hat{\theta}_j(\wvec))_{\text{AMIP}} \gets (\team_i(\onevec) - \team_j(\onevec)) + \sum_{n \in I} \Delta_n$}

    \STATE

    \STATE \rev{\(\triangleright\) \textbf{Compute the exact refit for verification.}}
    \STATE \rev{$\hat{\btheta}(\wvec) \gets \text{FitBTModel}(\mathcal{D}_{-\alpha})$}, where $\mathcal{D}_{-\alpha}$ is the set of points in $\mathcal{D}$ except those with indices in $I$
    \STATE \rev{Compute new difference: $(\team_i(\wvec) - \team_j(\wvec))$}

    \STATE

    \IF{\rev{$\text{sign}((\team_i(\onevec) - \team_j(\onevec))) \neq \text{sign}(\team_i(\wvec) - \team_j(\wvec))$}}
        \STATE \rev{\textbf{return}}
        \rev{``Arena is $\alpha$-level top-$\topk$ non-robust'', $(i,j), (\team_i(\onevec) - \team_j(\onevec)), (\team_i(\wvec) - \team_j(\wvec)), I$}
    \ENDIF

\ENDFOR

\STATE

\STATE \rev{\textbf{return} ``Arena was not found to be $\alpha$-level top-$\topk$ non-robust''}

\end{algorithmic}
\end{algorithm}

\paragraph{Runtime.} 
%The robustness-\revision{assessing} procedure we propose in this section 
%gives a very fast way to check the robustness of ranking systems for large databases of human preference data
\revision{The above procedure is fast for assessing the robustness of preference-based ranking systems. For example, we tested our method on historical preference datasets released by the Chatbot Arena project and hosted on Hugging Face \citep{chiang2024chatbot}.}
\revision{Specifically, we check} top-1 and top-5 robustness on a dataset of size \revision{around $50{,}000$ evaluations} in under $3$ minutes on a personal computer equipped with an Apple M1 Pro CPU at 3200 MHz and 16 GB of RAM.

\section{Experiments}
\label{sec:experiments}
Our analysis reveals that 1) dropping as little as $0.003\%$ of the evaluation data can flip the top-ranked model in popular LLM evaluation platforms (\cref{sec:experiments:sensitivity}), 2) crowdsourced human-evaluated systems are about as non-robust as AI-evaluated systems (\cref{sec:experiments:human_vs_llm}), 3) the LLM-generated responses of the dropped evaluations appear similar in content (\cref{sec:experiments:matches}), and 4) sensitivity depends on BT score margins (\cref{sec:experiments:score_margins}). Henceforth, for convenience, we use ``robustness'' as shorthand for robustness of a system's top-$\topk$ ranking to dropping a small fraction, $\alpha$, of the data.

\subsection{Data and Setup}
We run our robustness check on a variety of LLM Arenas, including Chatbot Arena \citep{chiang2024chatbot}, MT-bench \citep{zheng2023judging}, Search Arena \citep{miroyan2025search}, Webdev Arena \citep{vichare2025webdev}, and Vision Arena \citep{chou2025visionarena}. For more information about each arena, see \cref{sec:arenas}. Our analysis relies on historical preference datasets released by the Chatbot Arena project \citep{chiang2024chatbot} and publicly hosted on LMArena’s HuggingFace account. Each record represents a matchup consisting of two LLMs that answer the same prompt, the names of the two models, and the user label indicating preference for model A, model B, or a tie. Figure~\ref{figure:confint-cba-human} presents the Bradley--Terry scores of the top-10 models on Chatbot Arena.  

To compare the robustness of LLM arenas to more classical use cases of BT models, we also run our check on two sports datasets, namely NBA \citep{nbaelofivethirtyeight} and ATP tennis \citep{Sackmann2024TennisATP}. For details on the sports datasets, see \cref{sec:arenas}.

For each dataset, we assess top-$k$ robustness with $k \in \{1, 3, 5, 10, 20\}$, extending up to the maximum number of models present in the respective arena when fewer than $20$ models are present.
% \textbf{Experimental Setup.} Using this pairwise comparisons data, we construct a BT design matrix, treating each preference (match outcome) as a binary outcome. For more details on this setup, see \cref{sec:bt-setup}). We then compute model/player rankings based on full-data BT scores and run our method to check for top-$\topk$ robustness, for $\topk \in \{1, 3, 5, 10, 20\}$. 
%For each value of $\topk$, we start by running the robustness check at a level of $\floor{\alpha N} = 1$, and continue incrementing $\floor{\alpha N}$ by $1$ each time until we find the smallest $\floor{\alpha N}$ value for which the arena is $\alpha$-level top-$\topk$ non-robust. This allows us to report on the fraction of total evaluations that can be dropped to change the top-$\topk$ LLM rankings for a given arena.

\subsection{Sensitivity of LLM Arenas}\label{sec:experiments:sensitivity}
\begin{table*}[ht]
% \small
\centering
\begin{tabularx}{\linewidth}{>{\centering\arraybackslash}X c >{\centering\arraybackslash}X >{\arraybackslash}X}
%\begin{tabularx}{\linewidth}{lcccc}
\hline
\textbf{Arena} & \textbf{Evaluator (Judge)} & \textbf{Number Dropped} & \textbf{Percentage Dropped} \\
\hline
Chatbot Arena & Human & 2 out of 57477 & 0.00348\% \\
Vision Arena & Human & 28 out of 29845 & 0.0938\% \\
NBA Games & NA & 17 out of 109892 & 0.0155\% \\
Chatbot Arena & LLM & 9 out of 49938 & 0.0180\% \\
Webdev Arena & Human & 18 out of 10501 & 0.171\% \\ 
Search Arena & Human & 61 out of 24469 & 0.253\% \\
\rowcolor{gray!20} MT-bench & LLM & 40 out of 2400 & 1.67\% \\
\rowcolor{gray!20} ATP Tennis & NA & 6 out of 278 & 2.16\% \\
\rowcolor{gray!20} MT-bench & Human & 92 out of 3355 & 2.74\% \\
\hline
\end{tabularx} 
\caption{Results of checking top-1 robustness of BT-scores 
on each of the arenas, listed in ascending order of robustness (from the least to the most robust). The ``Number Dropped'' column reports the number of preferences (matches) that are sufficient to flip the first and second-place models (players). The ``Percentage Dropped'' column shows this number as a percentage of the number of total preferences in the full arena. Datasets we found to be robust at an $\alpha$-level of \(1\%\) are colored in gray.}
\label{table:nonrobustness-table}
\end{table*}
We find many popular LLM arenas to be incredibly sensitive to data-dropping (see \cref{table:nonrobustness-table}). In particular, we find that dropping just two (\(0.003\%\) of) evaluations is enough to change the top-ranked model on Chatbot Arena from GPT-4-0125-preview to GPT-4-1106-preview; see the two surfaced prompts and response pairs in \cref{sec:dropped-preferences}. We then find that dropping just three ($0.005\%$ of) evaluations can change one of the models in the top-5 rankings (the 5th and 6th-ranked models changed). \revision{Surprisingly, GPT-4-1106-preview participated in the most matchups across the entire arena and GPT-4-0125-preview also participated in a sizable number of matchups, as shown in \cref{fig:model_occurrences_cba}, suggesting that data-dropping sensitivity cannot be attributed to a small sample size alone.}

\rev{In addition to reporting rankings based on point-estimate BT-scores, LMArena reports an approximate ranking based on the end points of bootstrap confidence intervals (see \citet{lmarena2025leaderboard,lmarena2024notebook,chiang2024chatbot}).}
\rev{Even with the bootstrap-based rankings, we still find arenas to be surprisingly sensitive to worst-case data-dropping. For instance, we surface arenas where the bootstrap-based ranking outputs a single top-ranked model, but upon small-fraction data dropping, the model becomes no longer the sole top-ranked model (see \cref{figure:confint-webdev} in \cref{sec:sensitivity-confints}). See \cref{table:confint-nonrobustness-table} in \cref{sec:sensitivity-confints} for more details on the sensitivity of LMArena rankings based on bootstrap confidence intervals.}
%Dropping less than $0.05\%$ of the evaluation data is sufficient to change the top-1 and top-5 model rankings in Chatbot Arena. MT-bench, however, requires nearly $3\%$ of the data to be removed in order to change the top ranked model and over $3\%$ to change one of the models in the top-5.

Out of the LLM arenas we analyze, MT-bench is the sole benchmark that is robust at an $\alpha$-level of $0.01$ (see \cref{table:nonrobustness-table}). Here, dropping 92 out of $3{,}355$ ($2.74\%$ of) evaluations changes the top model from GPT-4 to Claude-v1. Dropping 110 ($3.28\%$ of) matchups can change one of the models in the top-5 rankings (again, the 5th and 6th ranked models changed). There are several reasons that may lead MT-bench to be \revision{much} more robust than the other LLM arenas. MT-bench consists of 80 carefully-designed multi-turn questions intended to differentiate models on core capabilities such as math, reasoning, and writing, and annotated by expert annotators \citep{zheng2023judging}. In contrast, all other arenas in our analysis are large-scale crowdsourced platforms, which rely on user-submitted prompts and crowd-sourced preference judgments.
\subsection{Humans vs. LLM-as-a-Judge}
\label{sec:experiments:human_vs_llm}
% jh: can highlight this section. is of importance to the alignment community.
\revision{Within arenas that used both human and LLM judges}, we find neither human-annotated nor LLM-annotated datasets to be clearly more sensitive than the other to worst-case data-dropping (see \cref{table:nonrobustness-table} and \cref{fig:human-vs-llm-judge}). 
% When asked what percentage of evaluations were required to change the top-$\topk$ model rankings 
For Chatbot Arena, we find that the human-annotated dataset is slightly more sensitive (required dropping fewer evaluations) for $k \in \{1, 5, 10, 20\}$ while the LLM-annotated dataset is slightly more sensitive for $k=3$ (see \cref{fig:human-vs-llm-judge}). In contrast, for MT-bench, the LLM-annotated dataset is more sensitive than the human-annotated dataset for all $k \in \{1, 3, 5\}$, perhaps due to the use of expert-human annotators.\footnote{We do not test $k \in {10, 20}$, as MT-bench includes only six models.} Taken together, we cannot conclude that rankings based on human preferences nor those based on LLM-as-a-judge preferences are systematically more sensitive than the other. 
\begin{figure}[t]
    \centering
    \includegraphics[width=0.55\linewidth]{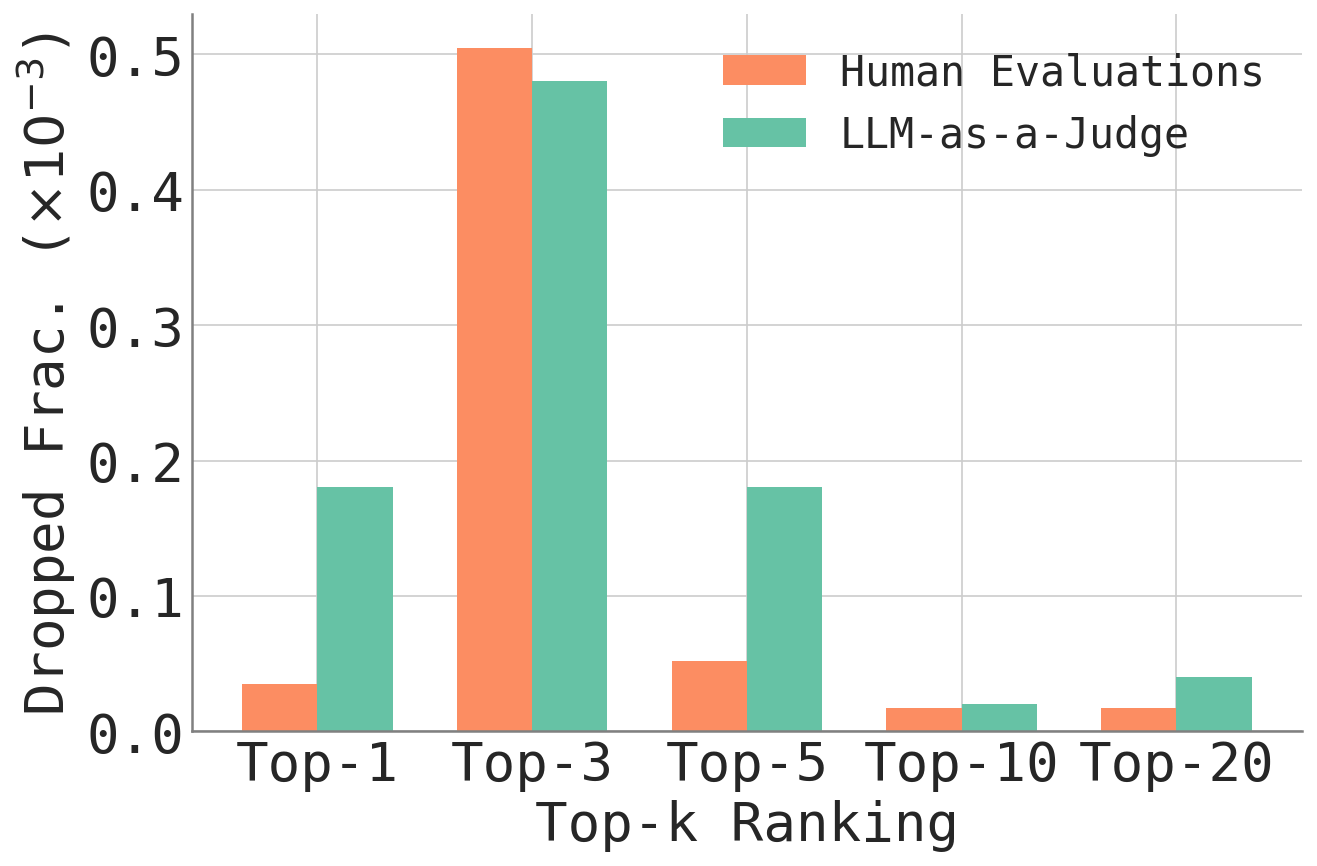}
    %\vspace{-0.42cm}
    \caption{Each bar shows the fraction of data points dropped from Chatbot Arena that is sufficient to \revision{demote the BT score of a} model inside the top-$\topk$ to outside of the top-$\topk$ ($k \in \{1, 3, 5, 10, 20\}$). The orange bars correspond to human evaluators and green bars to LLM-as-a-judge evaluators.}
    \vspace{-0.6cm}
    \label{fig:human-vs-llm-judge}
\end{figure}
% This finding agrees with previous work showing that strong LLMs, such as GPT-4, can closely approximate human preferences in model evaluation tasks \citet{zheng2023judging}. 
%In particular, \citet{zheng2023judging} demonstrate that GPT-4 achieves over $80\%$ agreement with expert human annotators on MT-bench. 
%Surprisingly, we find that human-judged arenas are slightly less robust than the LLM-judged arenas, indicating low signal in the human-annotated datasets. 
%This raises an interesting fundamental question: if the ``gold standard'' human annotations are noisier than the LLM-judge, then can we trust either source to provide a meaningful signal for ranking top-performing models?
\subsection{Inspecting Dropped Preferences}
\label{sec:experiments:matches}
Our method can identify the prompts and response-pairs responsible for changing top leaderboard rankings. On Chatbot Arena, we find that dropping just \textit{two} human evaluations suffices to flip the rankings of GPT-4-1106-preview (originally ranked first) and GPT-4-0125-preview (ranked second). We provide these prompts and response pairs in \cref{sec:dropped-preferences}. \rev{A qualitative analysis of the prompt-response pairs (see \cref{sec:dropped-preferences}) shows that the two surfaced preferences correspond to cases that a strong judge model (GPT-5.1) identifies as atypical (i.e., different to what the ``typical'' user might prefer).} In both cases, GPT-4-1106-preview was judged to have lost against substantially lower-ranked models: Vicuna-13b (ranked 43rd) and Stripedhyena-nous-7b (ranked 45th). Dropping these two anomalous losses is enough to raise GPT-4-1106-preview’s position from second to first.

\section{Related Work}
\label{sec:relatedwork}
\subsection{Vulnerabilities in AI Leaderboards}
\label{sec:adversarial-attacks}
Despite its ease-of-use and widespread popularity, large-scale, community-driven platforms like Chatbot Arena are found to be vulnerable to adversarial attacks that can distort model rankings. \citet{min2025improving} demonstrate that Chatbot Arena is vulnerable to vote-rigging: by injecting just a few hundred manipulated votes (out of 1.7 million), attackers can significantly change the top model rankings.  Similarly, \citet{huang2025exploring} find that an attacker can accurately identify which model produced a response on Chatbot Arena, and use that to systematically upvote or downvote a target model. They propose several defenses (e.g., authentication, rate limits, malicious-vote detection) that make the leaderboard more robust to adversarial agents. Injected votes may be especially easy to construct on LLM-as-a-judge systems, as recent works show that LLM judges can be gamed in systematic ways \citep{zheng2024cheating,raina2024llm}. Beyond vote-rigging, \citet{singh2025leaderboard} identify other issues such as data leakage and private testing practices that allow large, proprietary model developers to selectively report the best-performing versions of their models on the arena. \citet{zhao2025challenges} present a case study showing that model rankings can shift when a fraction of votes comes from apathetic or arbitrary annotators. Their analysis finds that replacing $10\%$ of votes with uniform $\{0,1\}$ labels can move two models by up to five ranks. In contrast, we do not alter votes but instead demonstrate that rankings can change by removing an alarmingly small fraction ($0.003\%$) of the votes. More importantly, while \citet{zhao2025challenges} present a case study focused on the rankings of three specific test models, we develop a systematic method to evaluate the robustness of BT-based ranking systems under worst-case data dropping, which also identifies the specific prompt–response pairs driving ranking flips. \textcolor{black}{Beyond pairwise-preference-based rankings, past works have pointed out the fragility of LLM ranking systems based on absolute performance scores. \citet{perlitz2024efficient} show that rankings derived from the Mean Win Rate (MWR) metric (a variant of the Borda count) can be gamed to increase the ranking of a target model by evaluating numerous models that are almost equal to, but slightly weaker than, the target model, thereby inflating the target model's ranking. The authors suggest allocating more evaluation resources for models ranked at the top of the leaderboard, where the rankings may carry more weight.} Finally, while all \textcolor{black}{pairwise-preference-based} related works in this section focus on Chatbot Arena, we extend our analysis to other domains (vision, web design, search, and multi-turn dialogue) and find the leaderboard rankings on these platforms to be similarly non-robust.

\subsection{Data-dropping Robustness} % worst-case
\label{sec:data-dropping-robustness}
A growing body of works in statistics and theoretical computer science develops algorithms for assessing whether data analyses are robust to dropping a small, worst-case fraction of the data \citep{broderick2023automatic,kuschnig2021hidden,moitra2022provably,freund2023towards,nguyen2024mcmc,huang2024approximations,rubinstein2024robustness}. To our knowledge, only one prior work has investigated this question in the context of ranking systems: \citet{shiffman2023could} study the robustness of rankings in gene set enrichment analysis, showing that dropping just a few cells can alter the ranking of p-values derived from the hypergeometric test. In contrast, our work examines ranking robustness in a BT-based ranking system. While \citet{shiffman2023could} analyze p-value rankings, we analyze preference-based rankings of LLMs, extending AMIP \citep{broderick2023automatic} to study the robustness of BT-based ranking systems.

\section{Discussion}
\label{sec:discussion}
Crowdsourced LLM evaluation platforms like Chatbot Arena offer a way to rank LLMs by aggregating preferences over responses to open-ended prompts. There is good reason that this setup has been widely-adopted: it is easy to scale, doesn't require expert annotators, and enables the aggregation of many prompts and judgments across a wide range of users \citep{zheng2023judging,don2024future}.

In theory, this aggregation helps average out individual annotator variability and yields a signal that is generalizable. However, in practice, we find that model rankings can depend on just a small handful of human (or LLM) evaluations. Thus, we encourage users of leaderboards and benchmark contests to run our method to investigate the fragility of crowdsourced LLM evaluation platforms before publishing results. 

\rev{Sensitivity to worst-case data-dropping is often indicative of low signal-to-noise in the underlying data \citep{broderick2023automatic}; to help increase signal-to-noise, we recommend three different design-related improvements that AI arenas could make. (1) Collect richer forms of feedback beyond binary preferences (e.g.,} asking for evaluators' confidence levels \citep{mendez2022eliciting}.\footnote{The weighted logistic regression model used by Chatbot Arena can easily be extended to take in confidence ratings on top of binary preferences. One could imagine implementing this through encoding the confidence rating as a weight in the Win-Counts matrix described in the ``Chatbot Arena Leaderboard Calculation (Bradley--Terry model)'' Colab notebook.} \rev{(2) Design more discriminative prompts. Arenas could incorporate a prompt-filtering system to identify and remove uninformative prompts, or} create tools to identify prompts requiring specialized knowledge in order to route them to appropriate evaluators \citep{don2024future}. \rev{\citet{chiang2024chatbot} perform topic-modeling of the prompts submitted to Chatbot Arena. Their top-16 topics include ``Poetry Writing Prompts'' and ``Movie Recommendations and Ratings.'' The subjective nature of such topics may make differentiation between top models less meaningful. (3) Ensure higher-quality preference annotations. Arenas could use} mediators to perform fine-grained assessments of crowdsourced responses \citep{don2024future}, and categorize prompts by instruction type (e.g., factual recall, creative generation) to promote more fine-grained model comparisons within categories \citep{chia2023instructeval}.

A complementary line of work on creating high-quality synthetic benchmarks argues that separability---requiring performance gaps between models to be wide enough for leaderboard trends to remain stable under subsampling---should be a main design criterion \citep{li2024autobencher}. At the same time, our findings may suggest that apparent leaderboard differences may be artifacts of noise in the evaluation process rather than genuine performance gaps, which cautions against treating AI leaderboard rankings as definitive indicators of differences in model performance.

\subsubsection*{Acknowledgments}
\label{sec:ack}
This work was supported in part by an ONR Early Career Grant, the MIT-IBM Watson AI Lab, the NSF TRIPODS program (award DMS-2022448), a MachineLearningApplications@CSAIL Seed Award, and the Amazon AI Research Innovation Fellowship. We thank Hao Sun from the University of Cambridge and Google Deepmind for helpful initial discussions and for pointers to references.

\bibliography{references}
\bibliographystyle{iclr2026_conference}

\appendix
\section*{Appendix}
%\label{sec:appendix}
% \section{Code}
% \label{sec:code}
% Our code is publicly available at \url{https://github.com/JennyHuang19/IsRankingRobust}, including all scripts required to run our robustness auditing method and to reproduce the results presented in this paper.

% \section{Related Work Continued.}
% \label{sec:related-work-continued}
% \subsection{Limitations of the Bradley--Terry Model.}
% \label{sec:bt-critique}
% % https://arxiv.org/html/2408.10075v1
% % https://arxiv.org/html/2407.14477v2#:~:text=,PMLR%2C%202023
% While the Bradley--Terry model has been widely adopted in large-scale alignment, previous works have found that BT-based reward models ``fall short in expressiveness'' in modeling human-feedback. \citet{poddar2024personalizing} examines the performance of the BT-model at preference prediction under injected label noise (flipping $25\%$ and $50\%$ of preference labels), finding the BT-model to be fragile under this flipping of a subset of preferences. \citet{zhang2024beyond} show that cyclic preference structures cause BT models to fail because BT cannot represent cycles. \jenny{work in progress. reading through more works.} While these works examine limitations of the BT-model, they do so in the context of reinforcement learning with human feedback (where BT-scores are assigned to prompts) our work demonstrates its limitations in ranking (where BT-scores is assigned to models).

% \subsection{LLM Benchmark-Saturation.}
% \label{sec:benchmark-saturation}
% % https://arxiv.org/pdf/2407.08351

\section{\rev{Uncertainty Quantification}}
\label{sec:uncertainty-quantification}
\subsection{\rev{Sensitivity of LLM Arena Rankings Based on Bootstrap Confidence Intervals}}\label{sec:sensitivity-confints}
\rev{In addition to reporting rankings based on point-estimate BT-scores, LMArena reports an approximate ranking based on the end points of bootstrap confidence intervals (see \citet{lmarena2025leaderboard,lmarena2024notebook,chiang2024chatbot}). Specifically, \citet{chiang2024chatbot} computes bootstrap-confidence-interval-based rankings, which we will henceforth refer to as \textit{bootstrap-based rankings}, as
\begin{equation}
\begin{aligned}
    R_m = 1 + \sum_{m' \in [M]} \mathbf{1}\!\left\{ \inf C_{m'} > \sup C_m \right\},
\end{aligned}
\label{eq:bootstrap-based-ranking}
\end{equation}
where \(R_m\) denotes the rank and \(C_m\) the bootstrap confidence interval of model \(m\). Under this scheme, a model's ranking increases by one for every other model whose lower confidence-interval endpoint exceeds the upper endpoint of the model in question (see \cref{eq:bootstrap-based-ranking}). In other words, a model, $m$, is ranked below all models whose performance is significantly higher according to non-overlapping bootstrap confidence intervals.
This definition (see \cref{eq:bootstrap-based-ranking}) induces a set-valued ranking: multiple models may share the same ranking whenever their confidence intervals overlap with one another. Thus, a bootstrap-based ``rank’’ corresponds often to a set of statistically indistinguishable models, rather than a single model.}

\rev{In the bootstrap-based ranking setting, we follow the same notion of top-$k$ robustness introduced in \cref{defn:top-k-robustness}. An arena is deemed top-$k$ robust at level-$\alpha$ if no $\alpha$-fraction subset of data can be dropped to change the top-$k$ set of models. The only modification under the bootstrap-based ranking scheme is that each ``rank'' now corresponds to a set of statistically indistinguishable models. Thus, we regard the top-$k$ set as having changed whenever any model is added to or removed from this set.}

\begin{figure}[t]
    \centering

    % --- Top plot ---
    \begin{subfigure}{0.9\linewidth}
        \centering
        \includegraphics[width=\linewidth]{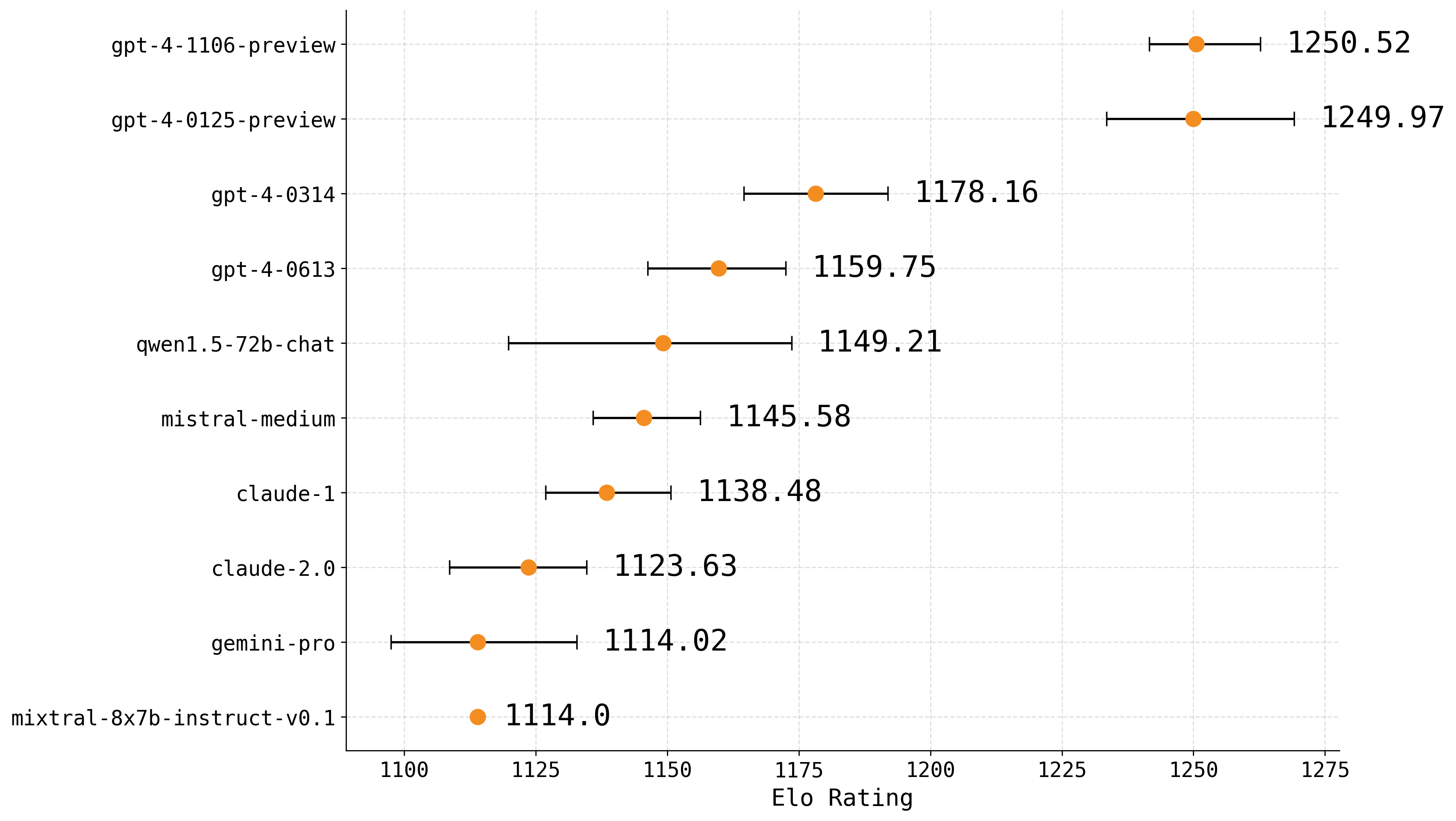}
        \caption{Original.}
    \end{subfigure}

    \vspace{0.5em} % space between plots

    % --- Bottom plot ---
    \begin{subfigure}{0.9\linewidth}
        \centering
        \includegraphics[width=\linewidth]{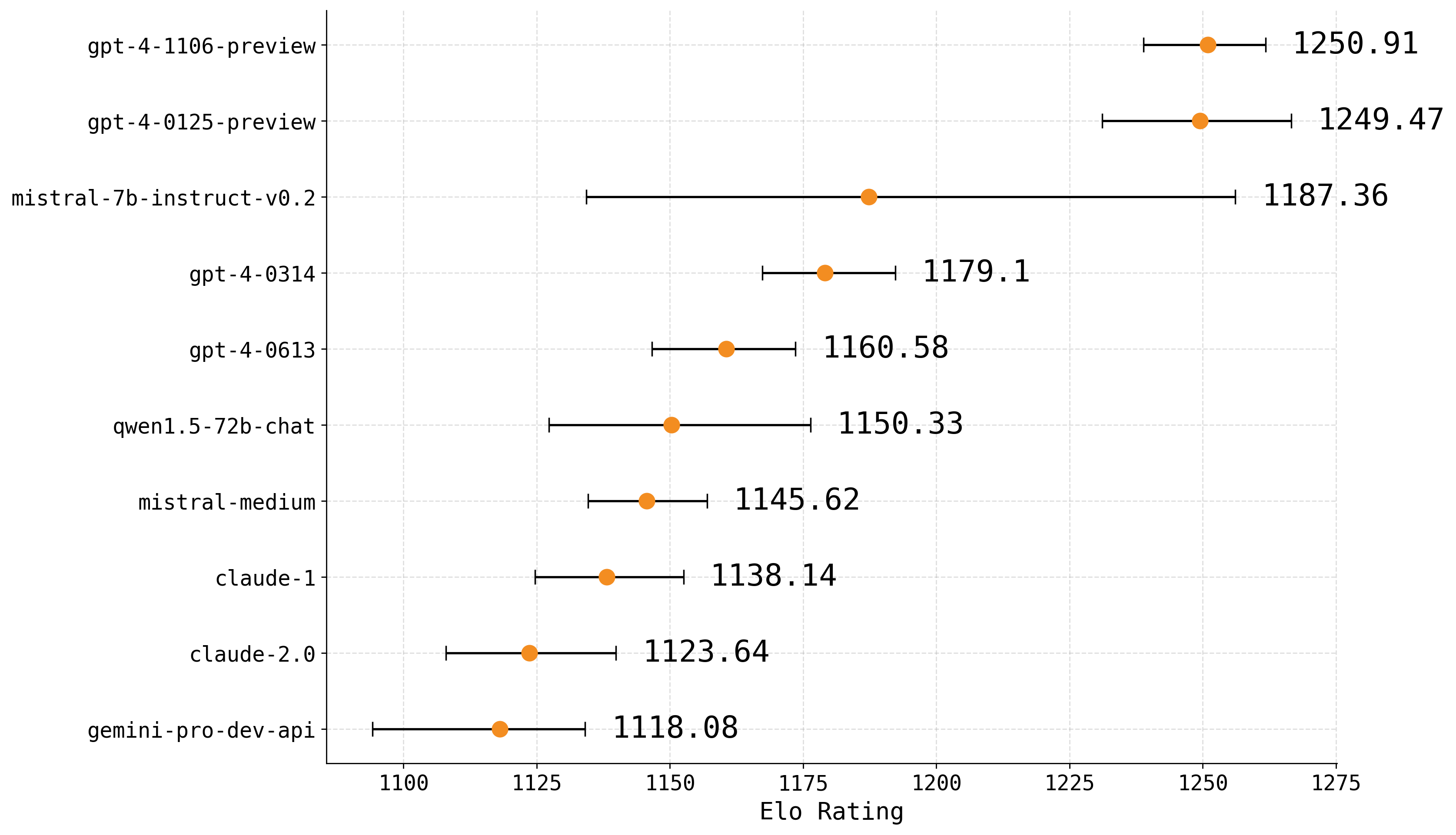}
        \caption{Post-data-dropping.}
    \end{subfigure}

    \caption{95\% Bootstrap-confidence-interval-based rankings on Chatbot Arena (Human Judge).}
    \label{figure:confint-cba-human}
\end{figure}

\begin{figure}[t]
    \centering

    % --- Top plot ---
    \begin{subfigure}{0.9\linewidth}
        \centering
        \includegraphics[width=\linewidth]{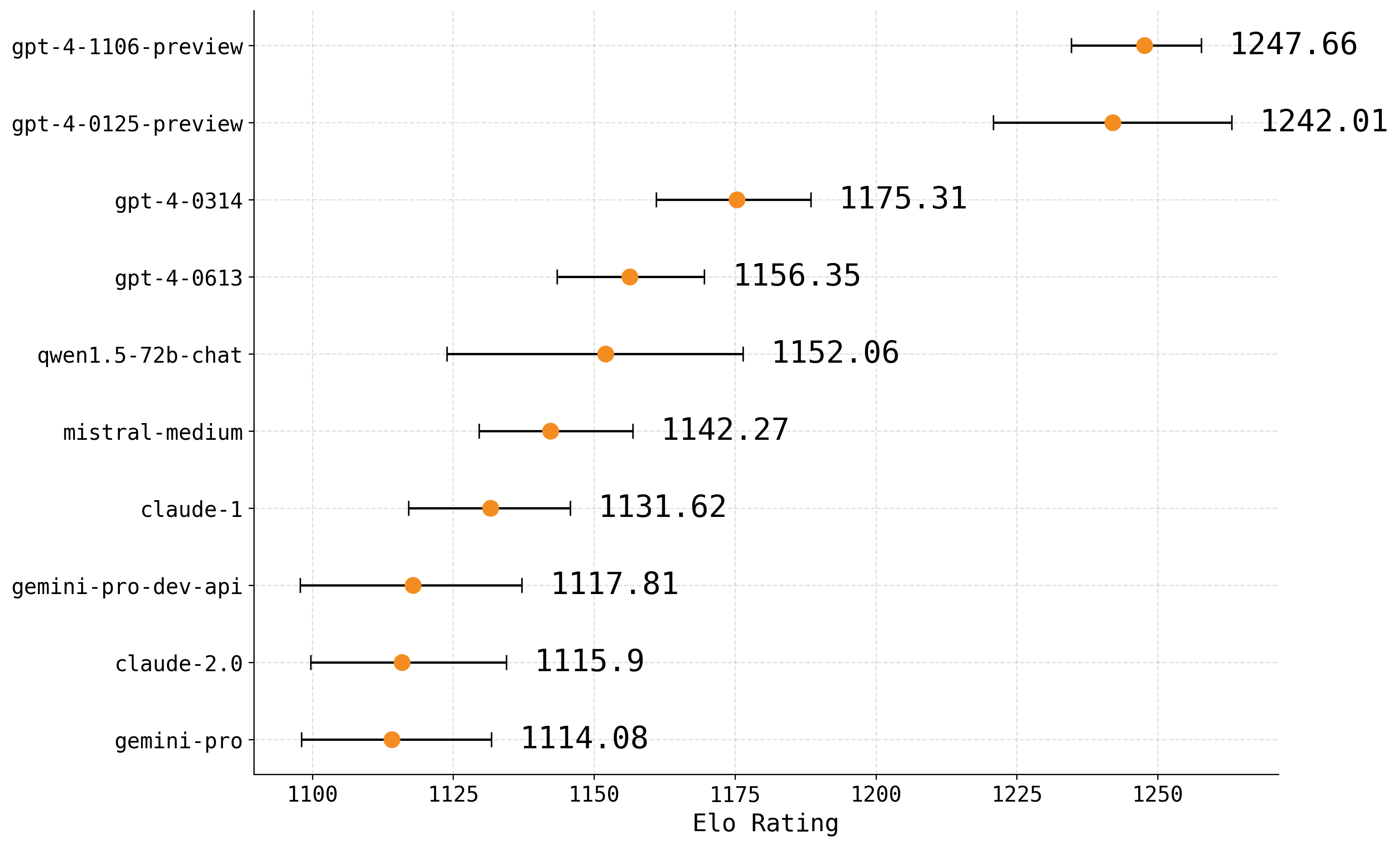}
        \caption{Original.}
    \end{subfigure}

    \vspace{0.5em} % space between plots

    % --- Bottom plot ---
    \begin{subfigure}{0.9\linewidth}
        \centering
        \includegraphics[width=\linewidth]{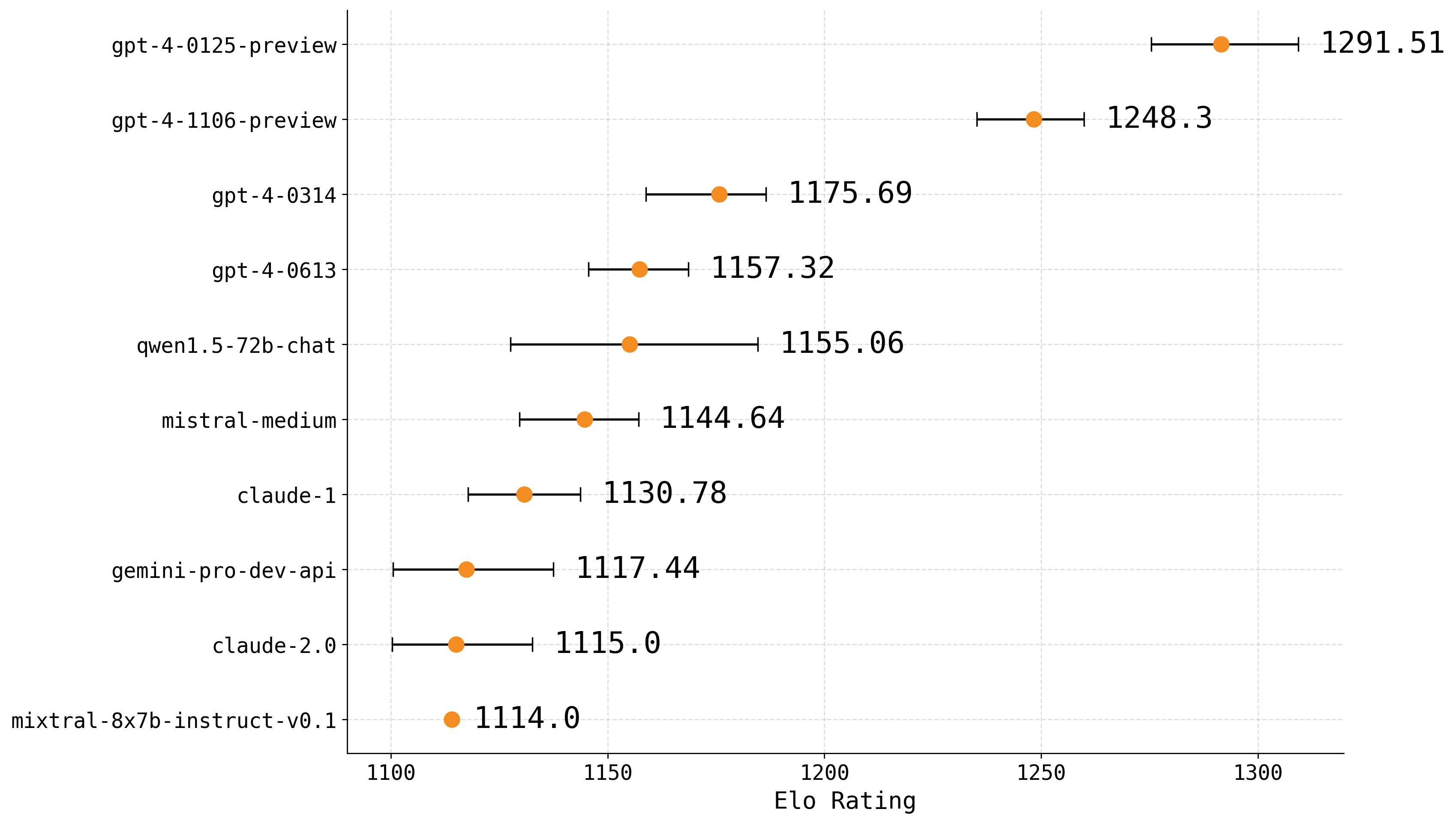}
        \caption{Post-data-dropping.}
    \end{subfigure}

    \caption{95\% Bootstrap-confidence-interval-based rankings on Chatbot Arena (LLM Judge).}
    \label{figure:confint-cba-llm}
\end{figure}

\begin{figure}[t]
    \centering

    % --- Top plot ---
    \begin{subfigure}{0.9\linewidth}
        \centering
        \includegraphics[width=\linewidth]{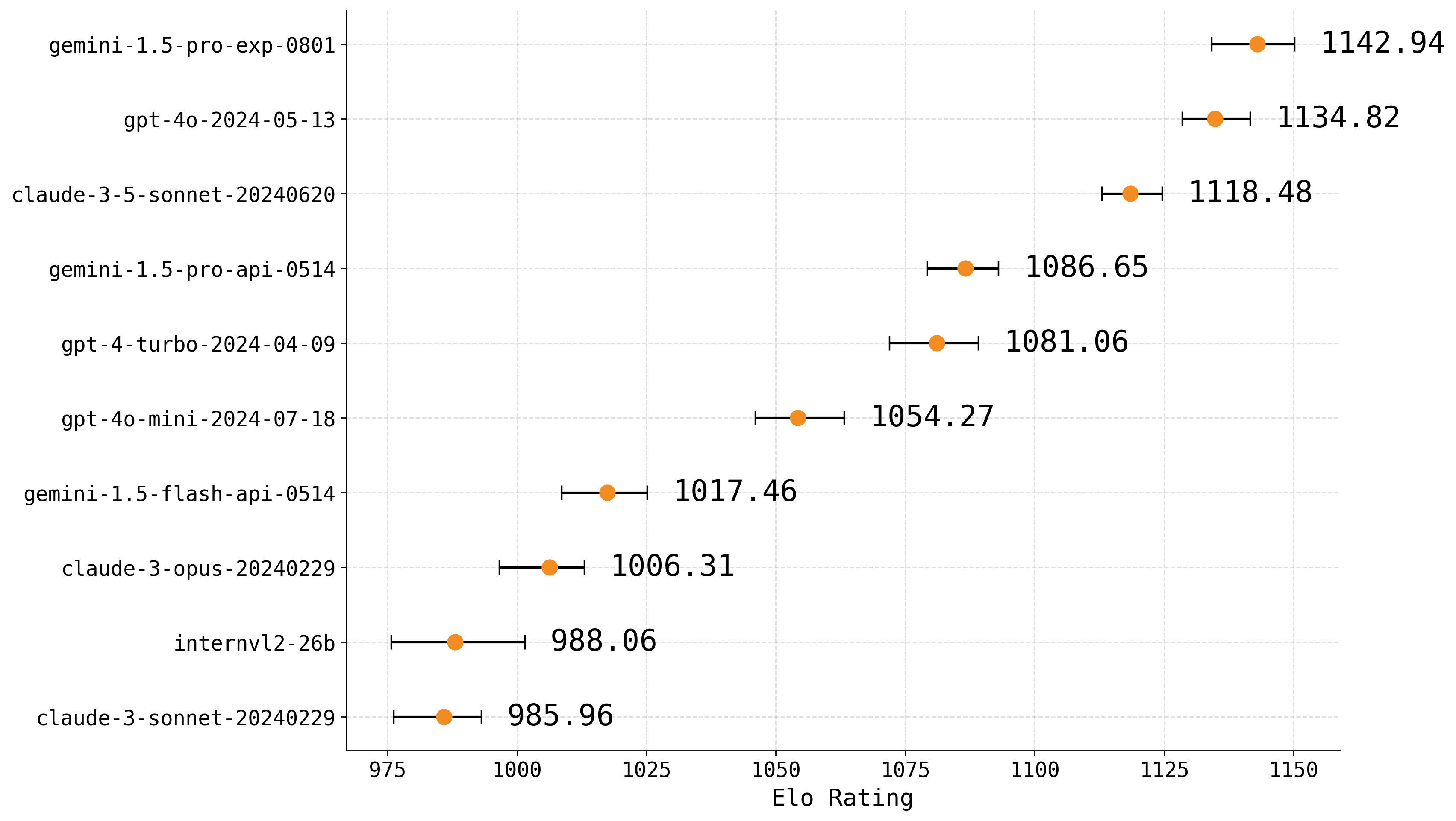}
        \caption{Original.}
    \end{subfigure}

    \vspace{0.5em} % space between plots

    % --- Bottom plot ---
    \begin{subfigure}{0.9\linewidth}
        \centering
        \includegraphics[width=\linewidth]{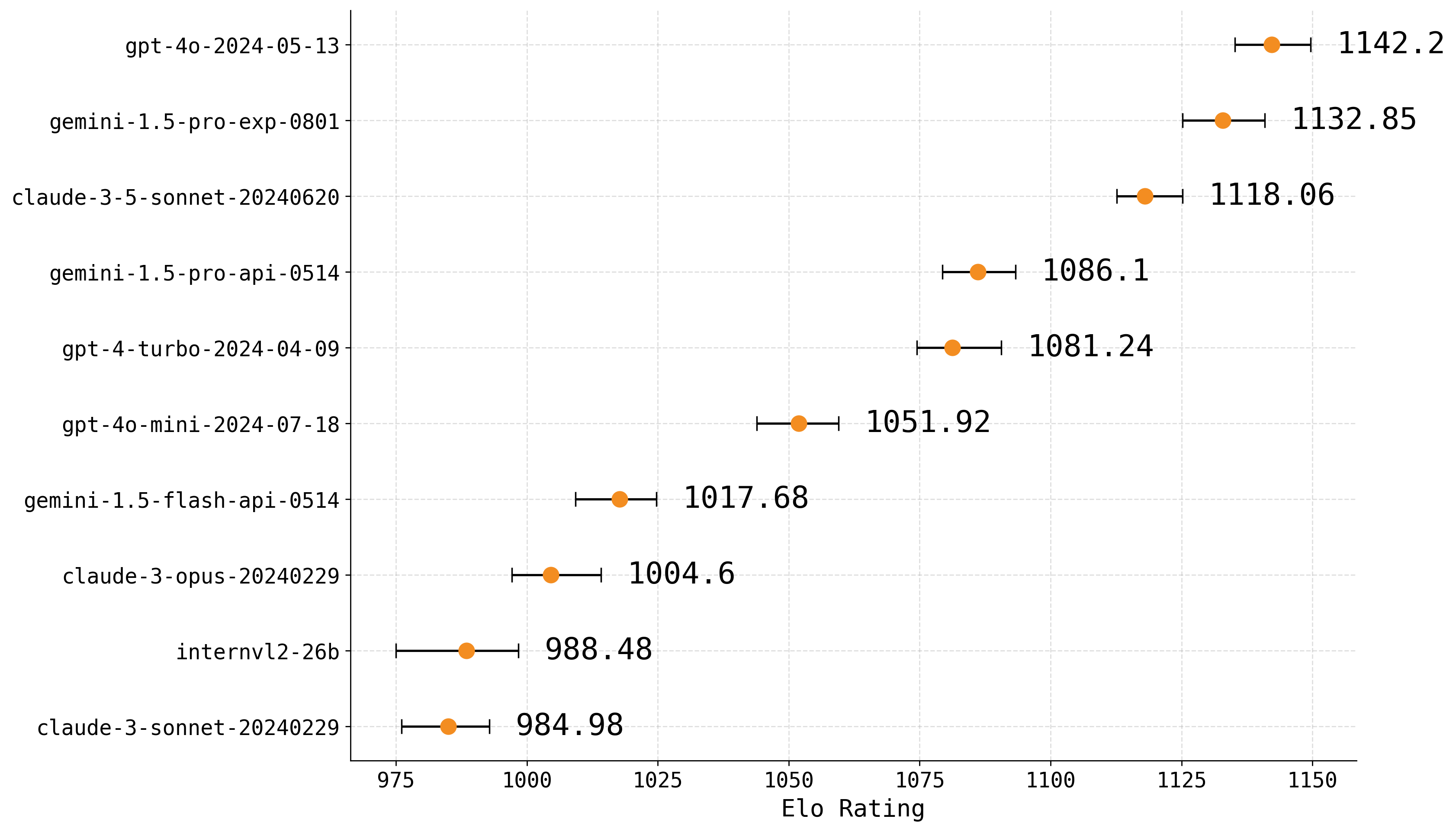}
        \caption{Post-data-dropping.}
    \end{subfigure}

    \caption{95\% Bootstrap-confidence-interval-based rankings on Vision Arena.}
    \label{figure:confint-vision}
\end{figure}

\begin{figure}[t]
    \centering

    % --- Top plot ---
    \begin{subfigure}{0.9\linewidth}
        \centering
        \includegraphics[width=\linewidth]{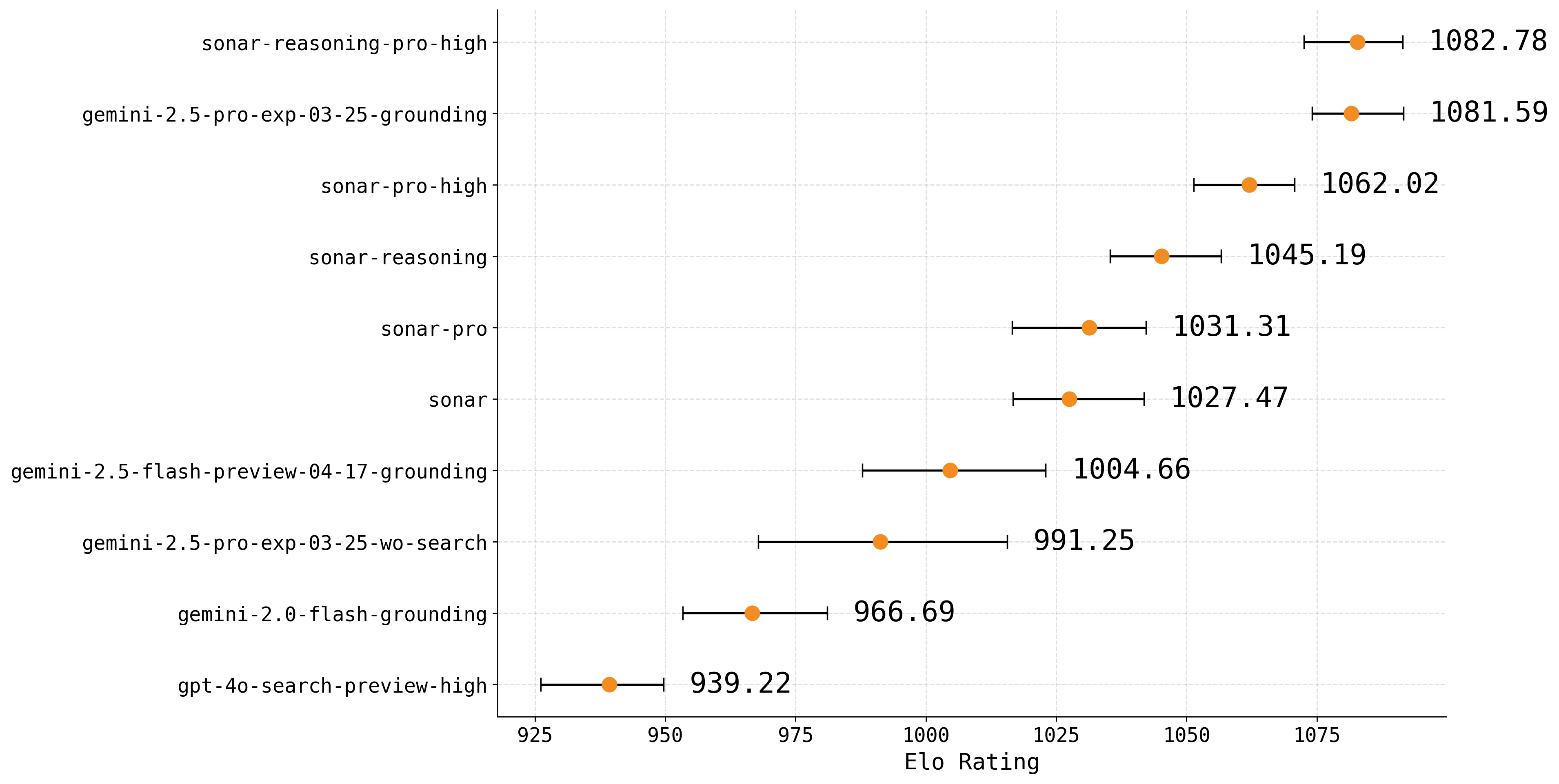}
        \caption{Original.}
    \end{subfigure}

    \vspace{0.5em} % space between plots

    % --- Bottom plot ---
    \begin{subfigure}{0.9\linewidth}
        \centering
        \includegraphics[width=\linewidth]{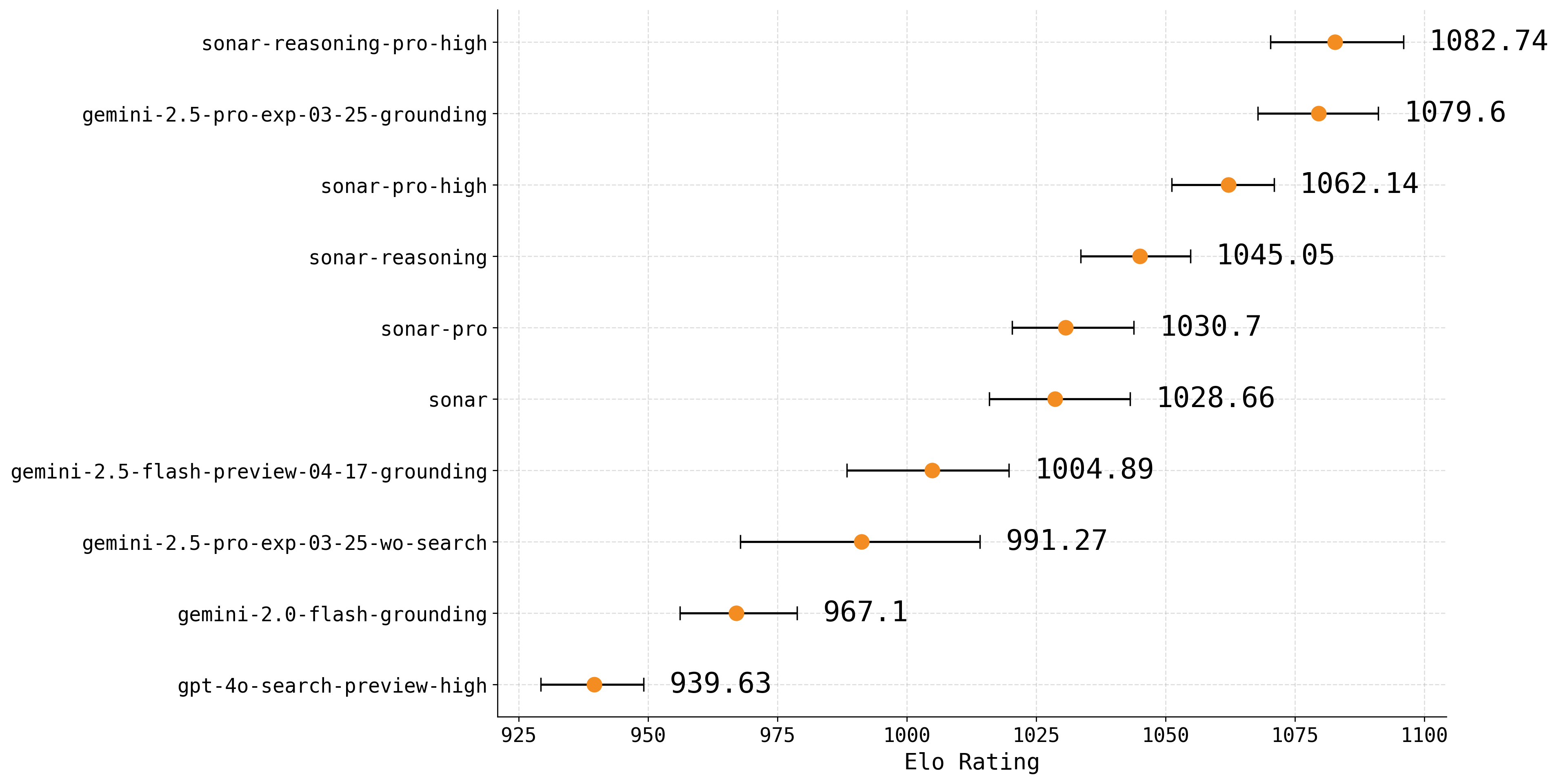}
        \caption{Post-data-dropping.}
    \end{subfigure}

    \caption{95\% Bootstrap-confidence-interval-based rankings on Search Arena.}
    \label{figure:confint-search}
\end{figure}

\begin{figure}[t]
    \centering

    % --- Top plot ---
    \begin{subfigure}{0.9\linewidth}
        \centering
        \includegraphics[width=\linewidth]{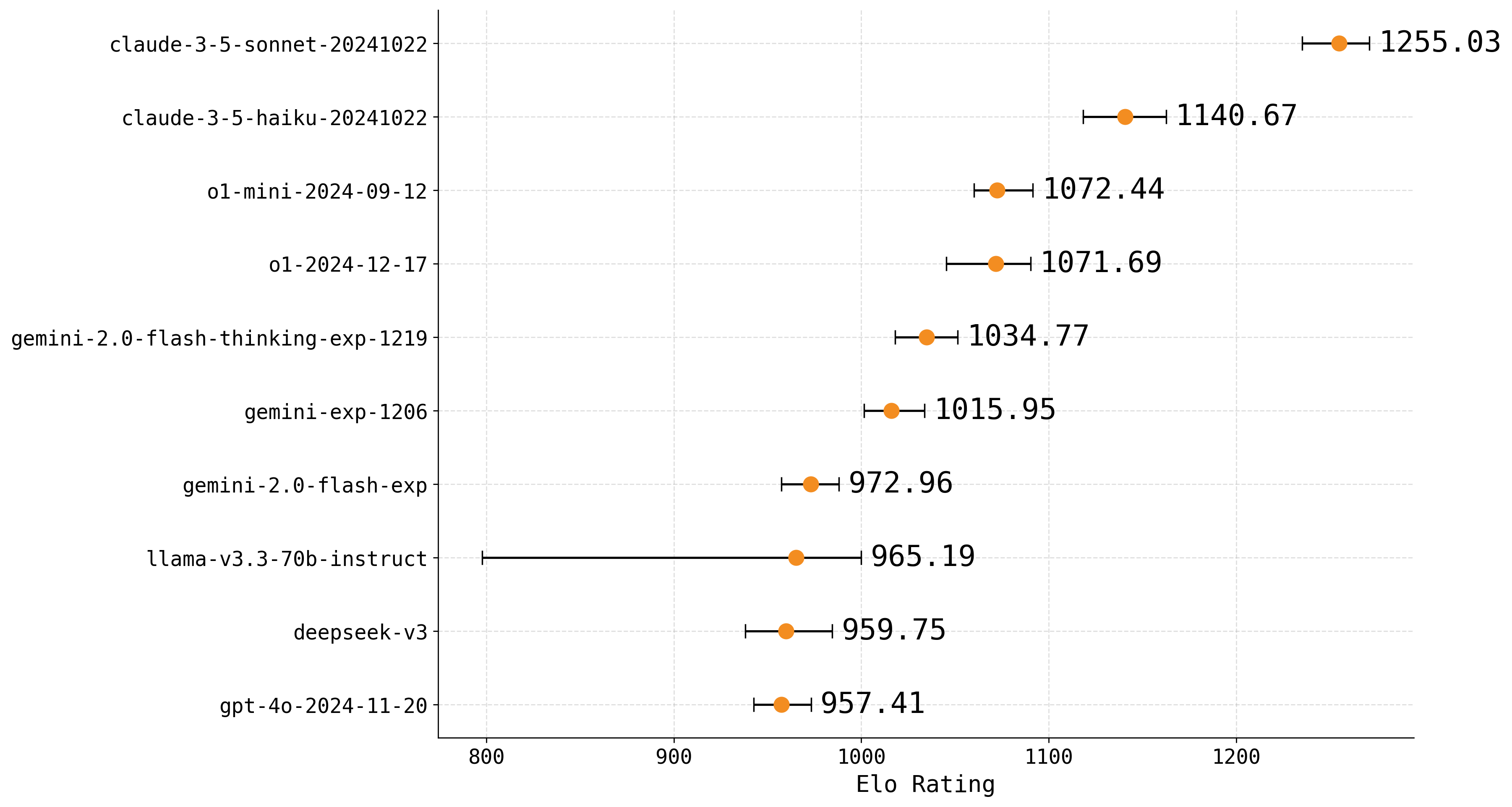}
        \caption{Original.}
    \end{subfigure}

    \vspace{0.5em} % space between plots

    % --- Bottom plot ---
    \begin{subfigure}{0.9\linewidth}
        \centering
        \includegraphics[width=\linewidth]{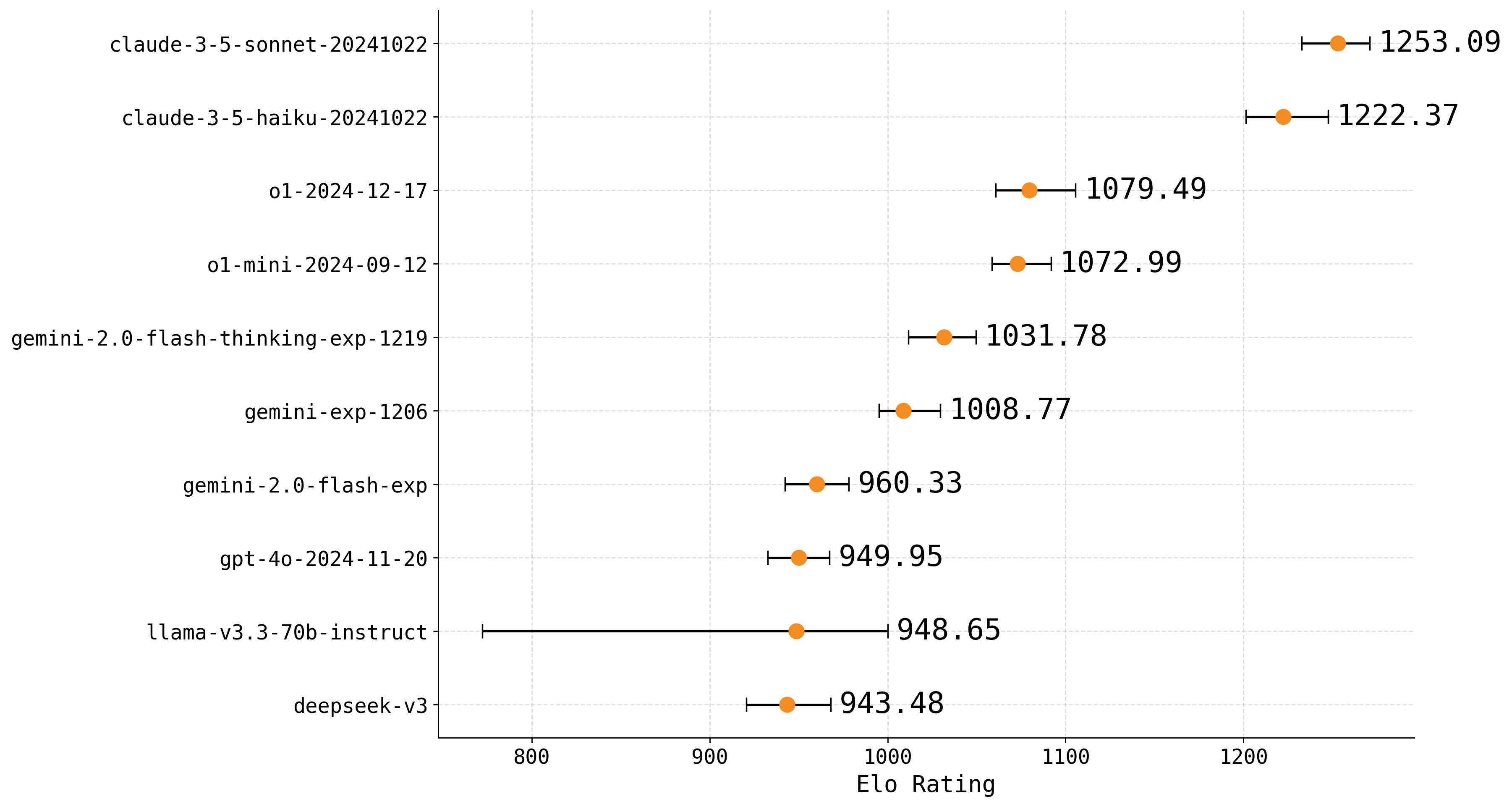}
        \caption{Post-data-dropping.}
    \end{subfigure}

    \caption{95\% Bootstrap-confidence-interval-based rankings on Webdev Arena.}
    \label{figure:confint-webdev}
\end{figure}

\begin{figure}[t]
    \centering

    % --- Top plot ---
    \begin{subfigure}{0.9\linewidth}
        \centering
        \includegraphics[width=\linewidth]{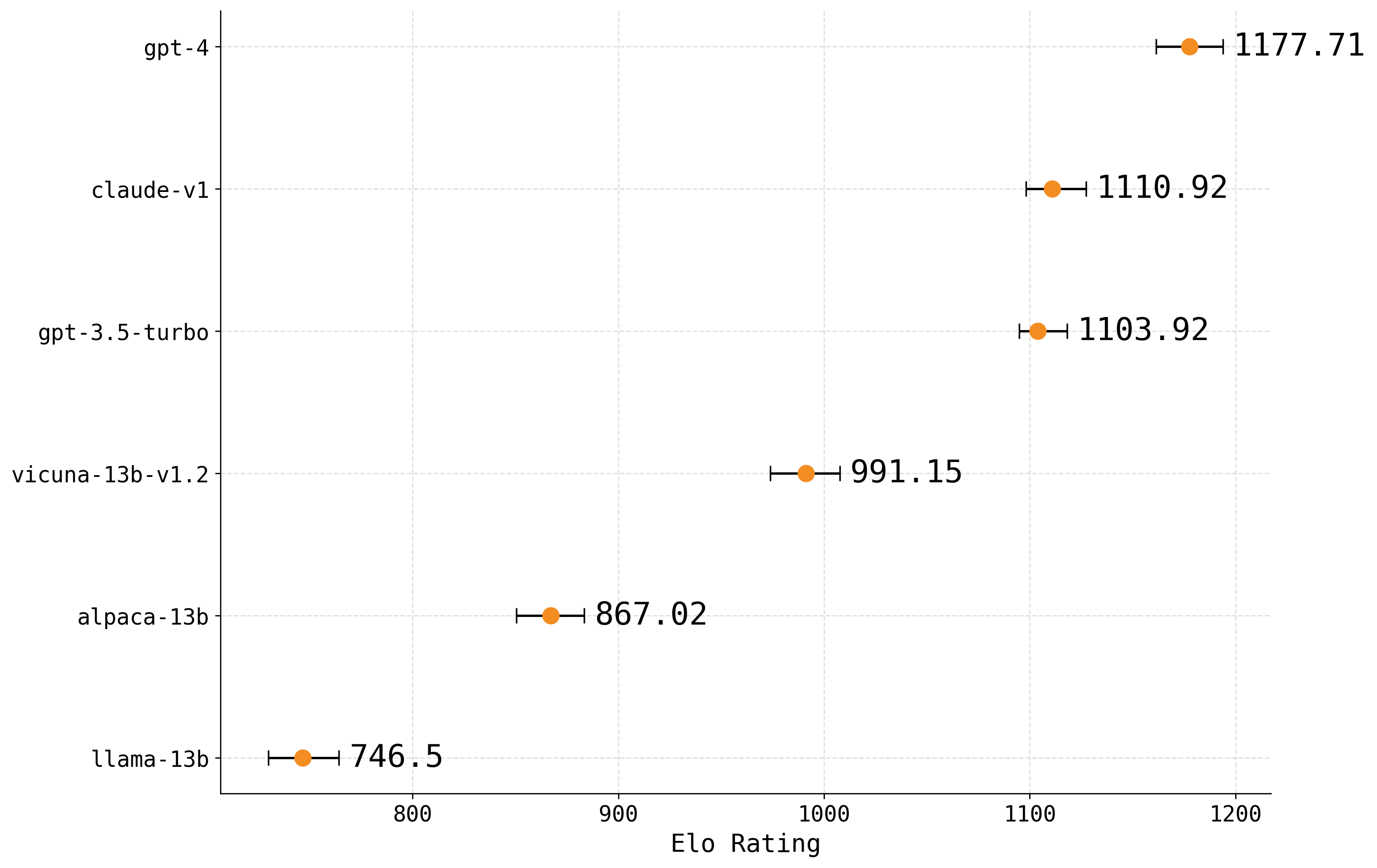}
        \caption{Original.}
    \end{subfigure}

    \vspace{0.5em} % space between plots

    % --- Bottom plot ---
    \begin{subfigure}{0.9\linewidth}
        \centering
        \includegraphics[width=\linewidth]{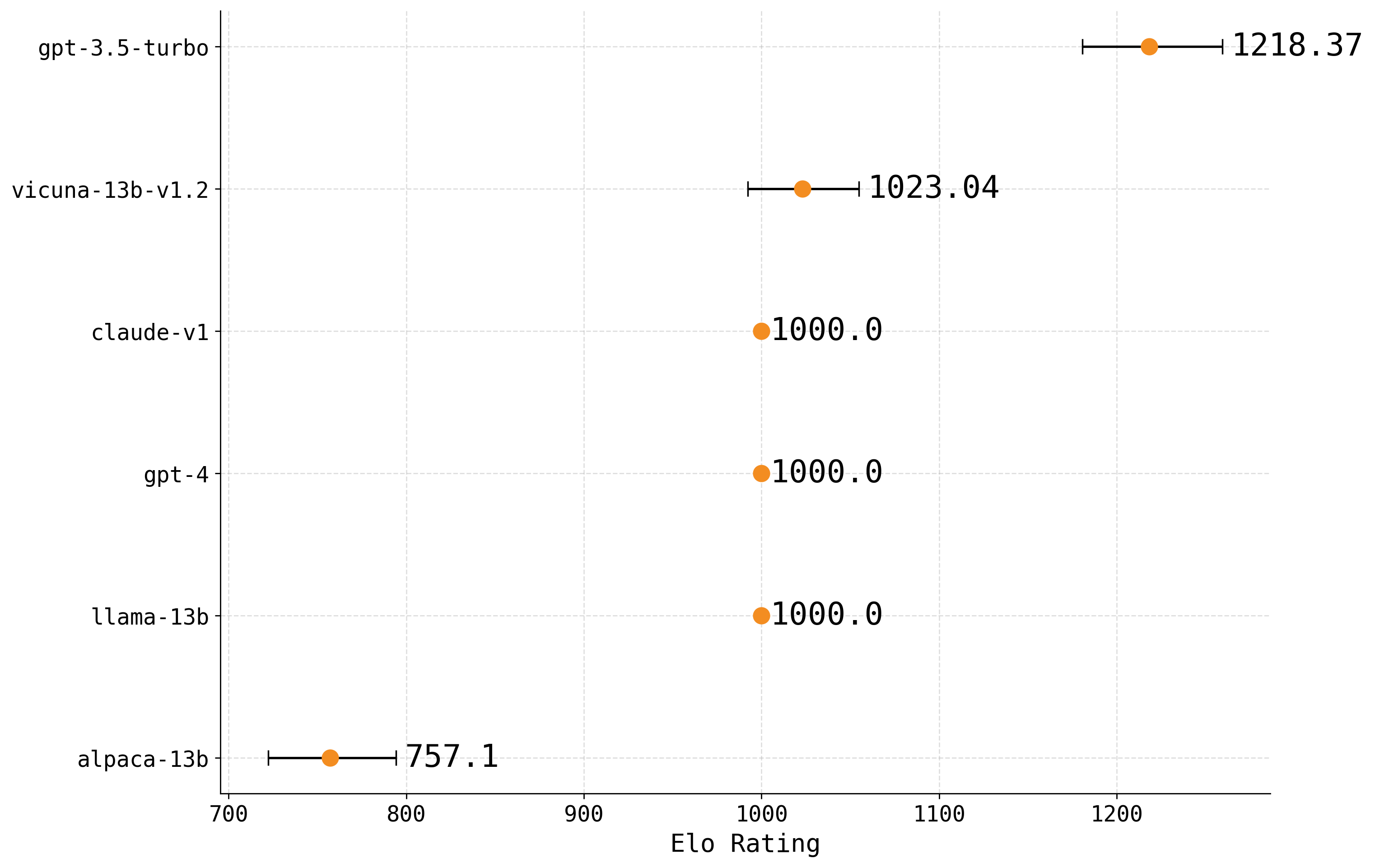}
        \caption{Post-data-dropping.}
    \end{subfigure}

    \caption{95\% Bootstrap-confidence-interval-based rankings on MTBench (Human Judge).}
    \label{figure:confint-MTBench-human}
\end{figure}

\begin{figure}[t]
    \centering

    % --- Top plot ---
    \begin{subfigure}{0.9\linewidth}
        \centering
        \includegraphics[width=\linewidth]{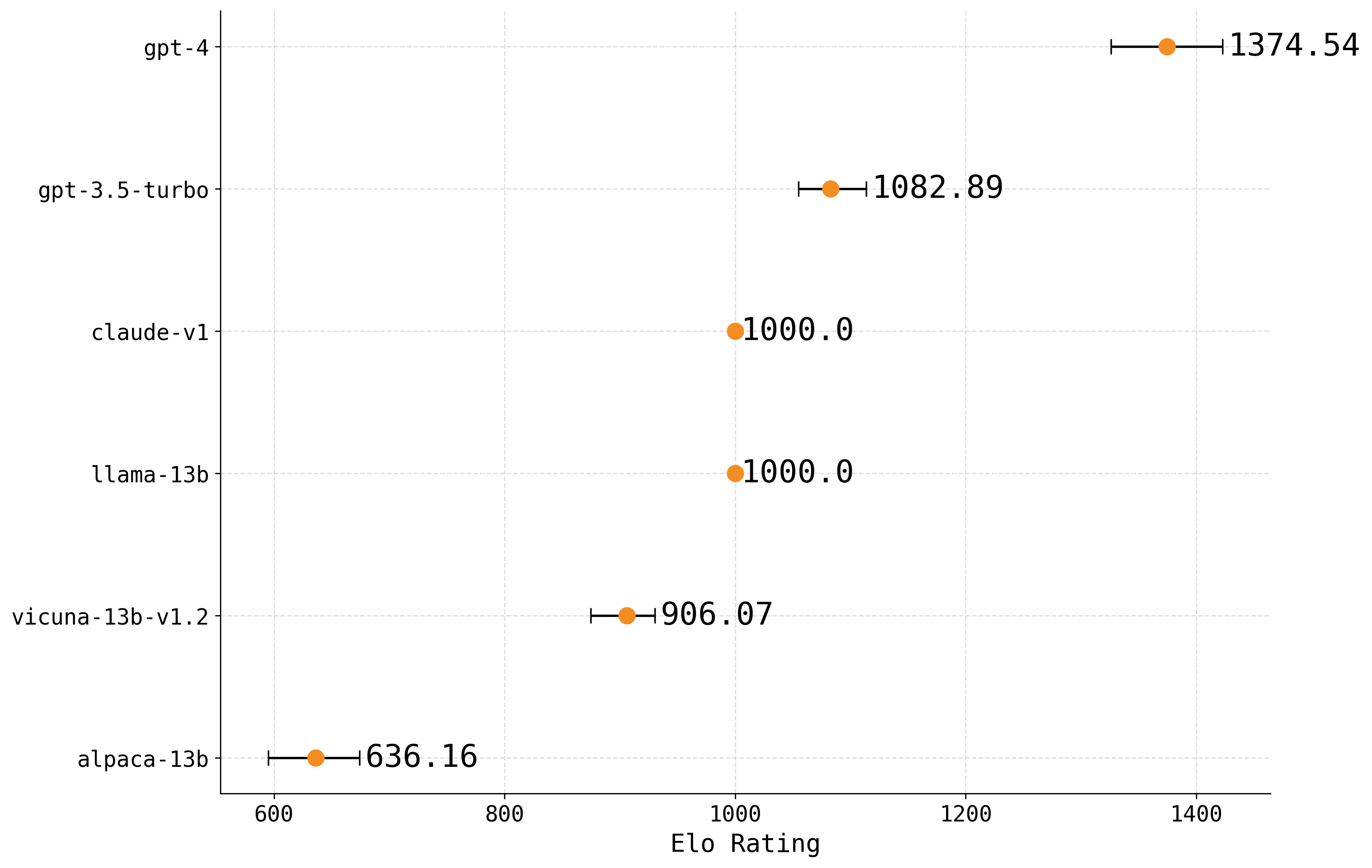}
        \caption{Original.}
    \end{subfigure}

    \vspace{0.5em} % space between plots

    % --- Bottom plot ---
    \begin{subfigure}{0.9\linewidth}
        \centering
        \includegraphics[width=\linewidth]{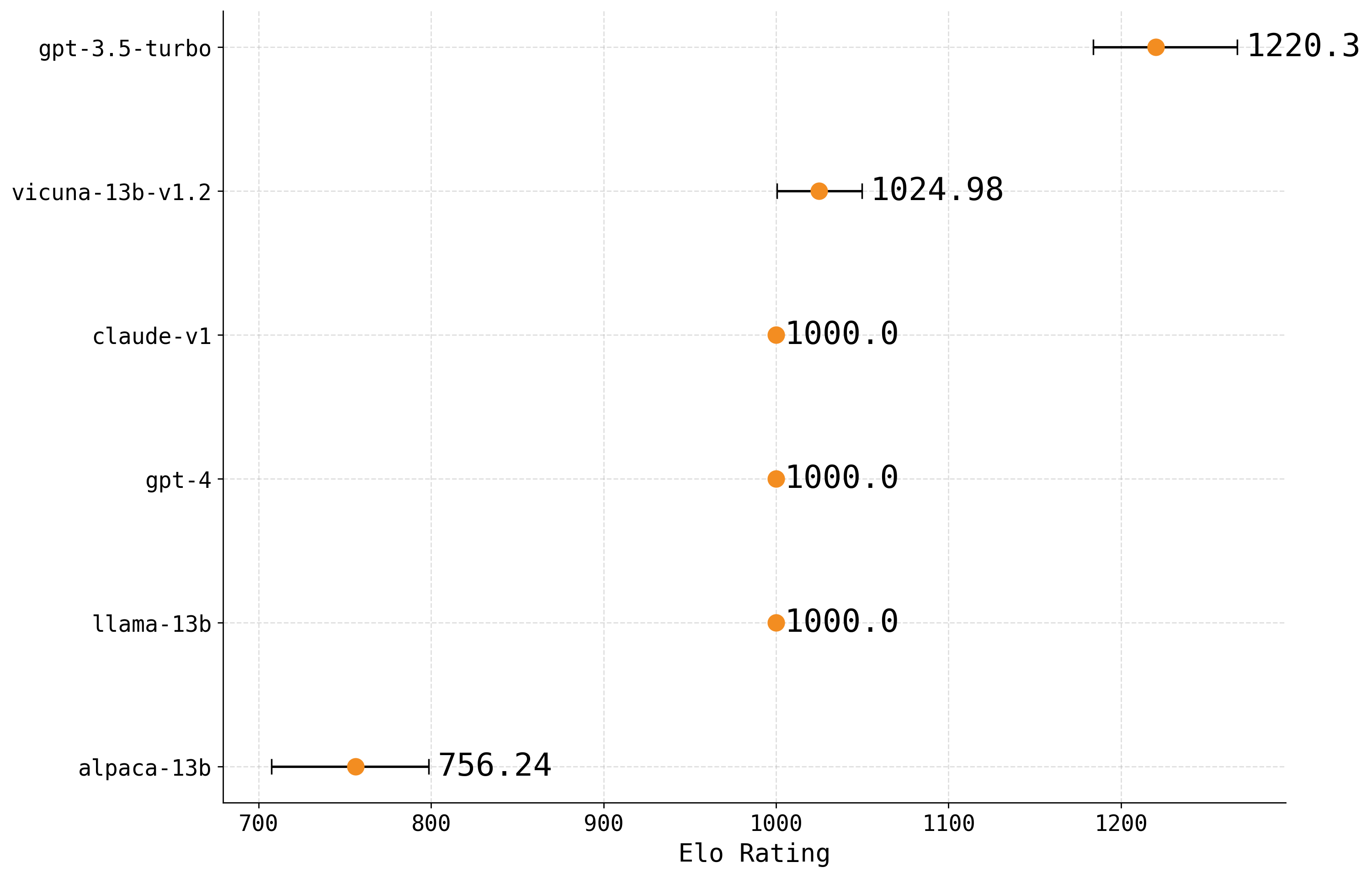}
        \caption{Post-data-dropping.}
    \end{subfigure}

    \caption{95\% Bootstrap-confidence-interval-based rankings on MTBench (LLM Judge).}
    \label{figure:confint-MTBench-llm}
\end{figure}

\begin{table*}[ht]
% \small
\centering
\rev{
\begin{tabularx}{\linewidth}{>{\centering\arraybackslash}X c >{\centering\arraybackslash}X >{\arraybackslash}X}
%\begin{tabularx}{\linewidth}{lcccc}
\hline
\textbf{Arena} & \textbf{Evaluator (Judge)} & \textbf{Number Dropped} & \textbf{Percentage Dropped} \\
\hline
Chatbot Arena & Human & 29 out of 57477 & 0.0510\% \\
Search Arena & Human & 25 out of 24469 & 0.103\% \\
Chatbot Arena & LLM & 75 out of 49938 & 0.150\% \\
Vision Arena & Human & 125 out of 29845 & 0.419\% \\
Webdev Arena & Human & 160 out of 10501 & 1.52\% \\ 
MT-bench & Human & 92 out of 3355 & 2.74\% \\
MT-bench & LLM & 40 out of 2400 & 4.00\% \\
\hline
\end{tabularx}
}
\caption{\rev{Results of checking top-1 robustness of bootstrap-based rankings on each of the arenas, listed in ascending order of robustness (from the least to the most robust). The ``Number Dropped'' column reports the number of preferences (matches) that are sufficient to flip the first and second-place models (players). The ``Percentage Dropped'' column shows this number as a percentage of the number of total preferences in the full arena.}}
\label{table:confint-nonrobustness-table}
\end{table*}

\rev{To construct \cref{table:confint-nonrobustness-table}, we first compute the bootstrap-based rankings on the full dataset, apply our method to identify influential preferences, remove those preferences, and then recompute the bootstrap-based rankings. Along with \cref{table:confint-nonrobustness-table}, we display the plots of the bootstrap-based rankings for the full data and the rankings post-data-dropping in \cref{figure:confint-cba-human,figure:confint-cba-llm,figure:confint-vision,figure:confint-search,figure:confint-webdev,figure:confint-MTBench-human,figure:confint-MTBench-llm}.} 

\rev{Despite the bootstrap's attempt to account for sampling uncertainty, we continue to find many arenas to be surprisingly sensitive to worst-case small-fraction data-dropping: the set of models ranked top-1 still changes in many arenas after removing a very small fraction of the arena. Across these experiments, we observe several arenas in which a new model enters the top-1 set (\cref{figure:confint-cba-human,figure:confint-vision,figure:confint-search}) and one arena in which a model is removed from the top-1 set (\cref{figure:confint-cba-llm}), all from dropping less than $1\%$ of preferences on the arena. We also surface arenas where the bootstrap-based ranking outputs a single top-ranked model, but upon small-fraction data dropping, the model becomes no longer the sole top-ranked model (see \cref{figure:confint-webdev}).}

\rev{This result shows that AMIP-based non-robustness is not an artifact of ignoring statistical uncertainty captured by confidence intervals. Rather, even after incorporating bootstrap variability, the arenas continue to be AMIP sensitive.}

\subsection{\rev{Distinction Between Worst-case Data-dropping Sensitivity and Confidence Intervals}}\label{sec:wcdatadropping-vs-confints}
\rev{Confidence intervals, such as the bootstrap intervals reported on LMArena \citep{lmarena2025leaderboard}, do quantify a form of sensitivity of BT-estimated rankings to variability across samples. However, the sampling-based sensitivity that bootstrap confidence intervals capture is conceptually different from that captured by AMIP. Bootstrap intervals characterize how much an estimate (e.g., the BT score) varies when data are resampled uniformly at random. In contrast, AMIP measures the maximum change in a BT-score difference that can be induced by removing a worst-case small fraction of the data. While frequentist \citep{gao2023uncertainty,hunter2004mm} confidence interval methods are meant to capture randomness in the data-generating process, the AMIP targets sensitivity of a single, fixed dataset. This focus on a single sample differs in spirit from the variability across ``counterfactual worlds'' that uncertainty quantification methods are meant to measure. In this sense, the two approaches answer complementary questions about the stability of a sample-based conclusion: the confidence intervals measure sampling uncertainty, while worst-case data-dropping robustness examines whether the conclusion is driven by a very small fraction of the observations in the sample. Although Bayesian credible intervals \citep{leonard1977alternative} also operate under the case of a single, fixed dataset, past work has demonstrated that data analyses can be both statistically significant in the Bayesian sense (credible interval does not include zero) and still sensitive to worst-case data dropping (see Bayesian hierarchical model case study in Section 4.4 of \citep{broderick2023automatic}). So, analogous to the frequentist case, the AMIP again represents a different and complementary check.}

\rev{These tools also differ in the statistical assumptions under which they provide guarantees. Bootstrap-based confidence intervals rely on the data being i.i.d. draws from a target population. Real-world preference datasets often depart from this regime due to differences in annotators (e.g., the same, or similar types of, annotators may annotate several prompts on LMArena), resulting prompt-selection biases, and various other potential context-based factors. AMIP, by contrast, does not require an i.i.d.\ assumption and therefore remains valid in settings where classical resampling tools do not apply reliably. Prior work \citep{broderick2023automatic} has demonstrated that data analyses can be simultaneously statistically significant yet worst-case data-dropping non-robust. In this sense, AMIP provides a complementary and practically useful lens for assessing the generalizability of LLM leaderboard rankings.}

\subsection{\rev{Uniform Data-Dropping Experiment}}
\label{sec:uniform-data-dropping}
\rev{To examine the contrast between worst-case data-dropping and dropping random pairwise preferences, we conduct a uniform subsampling experiment. For each arena, we drop $1\%$ of the evaluations uniformly at random, repeat the experiment 100 times, and record the fraction of runs in which the top-ranked model remains unchanged relative to the full arena.
For Chatbot Arena (human-judge), we additionally report robustness at a finer scale of $\alpha=0.1\%$.}

\begin{table*}[ht]
\centering
\rev{
\begin{tabularx}{0.9\linewidth}{l>{\centering\arraybackslash}X}
\hline
\textbf{Arena} & \textbf{Fraction of Trials Top-1 Robust} \\
\hline
Chatbot Arena (Human-judge) & 0.77 (0.97 at $\alpha = 0.1\%$) \\
Vision Arena & 1.00 \\
NBA Games & 1.00 \\
Chatbot Arena (LLM-judge) & 1.00 \\
Webdev Arena & 1.00 \\
Search Arena & 1.00 \\
MT-bench (LLM-judge) & 1.00 \\
ATP Tennis & 1.00 \\
MT-bench (Human-judge) & 1.00 \\
\hline
\end{tabularx}
}
\caption{\rev{Top-1 robustness of each arena under uniform-at-random data-dropping. Each entry reports the proportion of 100 trials in which dropping $1\%$ of the evaluations uniformly-at-random does not change the top-ranked model.}}
\label{table:uniform-dropping}
\end{table*}

\rev{The results in Table \ref{table:uniform-dropping} highlight a key conceptual distinction between uniform and worst-case data-dropping. Across nearly all arenas, dropping $1\%$ of the evaluations uniformly at random leaves the top-ranked model unchanged in every trial. Even Chatbot Arena (human-judge), which is the least stable under uniform subsampling, maintains its top-ranked model in $77\%$ of random $1\%$ deletions, a fraction that is many magnitudes larger than the $0.00348\%$ of preferences required to flip the top-ranked model when dropping the worst-case data subset. These results show that the rankings are extremely sensitive to dropping a worst-case small fraction of preferences, yet stable (at $\alpha=1\%$) to dropping preferences chosen at random. Taken together, these observations show that uniform and worst-case data-dropping probe fundamentally distinct failure modes.}

\section{\rev{Top-k Sets Can Be Characterized By Sets of Pairwise Player Comparisons}}
\label{sec:topk-equivalence-proof}

\rev{We show in \cref{prop:topk} that the top-$\topk$ set can be characterized by a set of pairwise player comparisons.}

\begin{proposition}
\label{prop:topk}
    Suppose we have $\totteams$ real numbers, $\allteams(\wvec):=\{\team_{\teamidx}(\wvec)\}_{\teamidx=1}^\totteams$. 
    %Define the top-$\topk$ set to be $\topkteamsgeneric:=\{ \team_{\teamidx}(\wvec): \rank(\team_{\teamidx}(\wvec); \allteams(\wvec))\le \topk\}$. 
    Suppose a set $\alttopkteams \subset \allteams(\wvec)$ satisfies $|\alttopkteams|=\topk$. Suppose it is the case that $\forall \; \team_{\teamidx}(\wvec) \in \alttopkteams$ and $\forall \; \team_{\altteamidx}(\wvec) \in \allteams(\wvec)\setminus\alttopkteams$, we have that $\team_{\teamidx}(\wvec) >\team_{\altteamidx}(\wvec)$. Then, it must be that $\alttopkteams$ is the top-$\topk$ set, \rev{i.e.,} $\alttopkteams=\topkteamsgeneric$.
\end{proposition}
\begin{proof}
    We first show that $\alttopkteams\subset \topkteamsgeneric$. Suppose that $\team_{\teamidx}(\wvec) \in \alttopkteams$. By assumption, we have that $\forall \; \team_{\altteamidx}(\wvec) \in \allteams(\wvec)\setminus\alttopkteams$,  $\team_{\teamidx}(\wvec) >\team_{\altteamidx}(\wvec)$. Since $|\allteams(\wvec)\setminus\alttopkteams|=\totteams-\topk$, there must exist at least $(\totteams-\topk)$ values in  $\allteams(\wvec)$ that are smaller than $\team_i(\wvec)$. This must mean that $\rank(\team_{\teamidx}(\wvec); \allteams(\wvec))\le \topk$, so $\team_{\teamidx}(\wvec)\in \topkteamsgeneric$ as needed.

    We next show that $\topkteamsgeneric \subset \alttopkteams$ by contradiction. Suppose there exists a $\team_{\altteamidx}(\wvec)$ such that $\team_{\altteamidx}(\wvec)\in \topkteamsgeneric$ but $\team_{\altteamidx}(\wvec)\notin \alttopkteams$. Since $\team_{\altteamidx}(\wvec)\notin \alttopkteams$, then $\team_{\altteamidx}(\wvec)\in \allteams(\wvec)\setminus\alttopkteams$. This means that $\forall \team_{\teamidx}(\wvec) \in \alttopkteams$ we have $\team_{\teamidx}(\wvec)>\team_{\altteamidx}(\wvec)$, and since $|\alttopkteams|=\topk$, this implies that $\rank(\team_{\altteamidx}(\wvec); \allteams(\wvec))>\topk$, contradicting the assumption $\team_{\altteamidx}(\wvec)\in \topkteamsgeneric$. 
\end{proof}

\section{AMIP approximation for BT Models}
\label{sec:BTlogistic}

\subsection{AMIP approximation of general weighted BT models}
For completeness we provide here a review on general AMIP approximation proposed by \citet{broderick2023automatic} to solve the optimization problem \cref{eqn:comb-optimization}.
%The key that makes the AMIP method fast for solving \cref{eqn:comb-optimization} is that it approximates the impact of dropping a data point\footnote{By ``impact'' here, we mean the impact of dropping a data point on some statistical quantity-of-interest, such as a regression coefficient, or a parameter, or a test prediction.} by using a first-order approximation (e.g., influence functions and variants). 

\citet{broderick2023automatic} propose relaxing $\wvec$ to allow continuous values and replacing the $\wvec$-specific quantity of interest with a first-order Taylor series expansion with respect to $\wvec$ around $\onevec$. This first-order Taylor series expansion is known as the \textit{influence function (IF)} approximation \citep{hampel2011robust}, a classic technique from robust statistics that approximates the affect of upweighting (or dropping) a data point on model parameters using a first-order Taylor series approximation in data-weight space. Influence functions have become popular tools for approximating resampling methods \citep{giordano2019higher} and assigning value to data that a model was trained on \citep{koh2017understanding,park2023trak}. This approximation applies to more general data analyses and quantities of interest.

In our case, this approximation amounts to replacing \cref{eqn:comb-optimization} with
%\revision{
\begin{align}
    \max_{\wvec \in W_{\alpha}} \sum\nolimits_{n=1}^N (1 - w_n) \left( \frac{\partial \hat{\theta}_i(\wvec)}{\partial w_n}\Big\vert_{\wvec = 1_N}- \frac{\partial \hat{\theta}_j(\wvec)}{\partial w_n}\Big\vert_{\wvec = 1_N} \right).
    \label{eq:linear-optimization}
\end{align}
Let 
\begin{equation}
    \begin{aligned}
        L(y_{n}, \theta) \coloneqq  w_{\win\loss}I_{y_{n}=\win} \log \sigma(\theta_{i_n}-\theta_{j_n})  + w_{\win\loss}I_{y_{n}=\loss} \log (1 - \sigma(\theta_{i_n}-\theta_{j_n})) \qquad\qquad\quad 
        \\+ w_{\tie}I_{y_{n}=\tie}\big(\log \sigma(\theta_{i_n}-\theta_{j_n})+\log (1 - \sigma(\theta_{i_n}-\theta_{j_n}))\big).
    \end{aligned}
    \label{eq:weighted-likelihood}
\end{equation}
to be the likelihood for a single data point. The impact of upweighting $\wvec$ on the parameter $\hat{\theta}_i(\wvec)$ is then given by
\begin{equation}
    \frac{\partial \hat{\theta}_i(\wvec)}{\partial w_n}\Big\vert_{\wvec = 1_N}
    \;=\; - H_{\hat{\theta}(1_N)}^{-1} \, \nabla_{\theta} L(y_n, \theta)\Big\vert_{\theta = \hat{\theta}(1_N)},
    \label{eq:derivative-expanded}
\end{equation}
where
\begin{equation}
\begin{aligned}
    H_{\hat{\theta}(1_N)} \;\coloneqq\; \frac{1}{N} \sum_{n=1}^{N} 
    \nabla^2_{\theta} L(y_n, \theta)\Big\vert_{\theta =\hat{\theta}(1_N)}.
    %\footnote{\(\theta\) is shorthand for \(\theta(\wvec)\) in \cref{eq:derivative-expanded,eq:hession}. \ys{no, they are not, they are generic argument of the function and we evaluate the fuction at $\hat{\theta}(w)$}}
\end{aligned}
\label{eq:hession}
\end{equation}
See \citet[Section 2.2.2]{broderick2023automatic} for more details on this derivation. In what follows we provide details on how to apply this approximation in BT models by reformulating it as a logistic regression.

\subsection{BT models as logistic regressions}
\textbf{Unweighted BT.} In the unweighted BT model with $w_{\win, \loss}=1, w_{\tie}=0$, with an abuse of data indices $n$, the preferences are assumed to be generated as 
\begin{equation}
    y_{n}\sim \text{Bernoulli}(\sigma(\theta_{i_n}-\theta_{j_n})),
\end{equation}
We can cast this model as a logistic regression with a specially-structured design matrix. We denote the corresponding ``design'' vector of the $n$th comparison, $x_{n}\in \{-1,0,1\}^{M}$, a vector encoding which two \revision{players} are being compared. That is, if the game is between players $i$ and $j$, then $x_n$ has a $1$ in the $i$th element, a $-1$ in the $j$th element, and $0$ otherwise. Using this structure, we can rewrite the model as a logistic regression model with $M-1$ parameters corresponding to the scores of the players, $\btheta=(\theta_1,  \dots, \theta_{M})\in \mathbb{R}^{M}$ with $\theta_1=0$,
\begin{equation}
    y_{n}\sim \text{Bernoulli}(\sigma(x_{n}^\top\btheta )).
\end{equation}
We fit the BT-model (i.e., estimate $\btheta$) by maximum likelihood of logistic regression,
\begin{equation}
\begin{aligned}
\hat{\btheta} := 
&\argmax_{\btheta:\theta_1=0} \;
\sum_{n=1}^N \big( 
    y_n \log \sigma(x_n^\top \btheta) \\
    &\qquad\qquad\quad + (1 - y_n) \log (1 - \sigma(x_n^\top \btheta))
\big).
\end{aligned}
\end{equation}

\textbf{Weighted BT.} The model actually used in e.g., ChatBot Arena that handles tie by 1) counting every winning/loss as two games with the same outcome and 2) couting tie as two games with opposite outcomes. This effectively sets $w_{\win, \loss}=2, w_{\tie}=1$. This special case can also be casted as a logistic regression with two copy of the design matrix same as unweighted version, $\bm{X}_{weighted}=[\bm{X}, \bm{X}]$. That is, suppose there are in total $N$ games, if the $n$th game is between players $i$ and $j$, then $x_{weighted, n}$ as well as $x_{weighted, n+N}$ has a $1$ in the $i$th element, a $-1$ in the $j$th element, and $0$ otherwise. The response $y_{weighted, n}=I_{y_{n}=\win}$ and $y_{weighted, n+N}=I_{y_n=\win}+I_{y_n=\tie}$. I.e., in the first copy of the game, a tie is counted as a loss and in the second copy of the game, a tile is counted as a win while winning and losing are counted twice in total from both copies. Then we can fit the weighted BT by maximum likelihood of logistic regression,
\begin{equation}
\begin{aligned}
\hat{\btheta} := 
&\argmax_{\btheta:\theta_1=0} \;
\sum_{n=1}^{2N} \big( 
    y_{weighted, n} \log \sigma(x_{weighted, n}^\top \btheta) \\
    &\qquad\qquad\quad + (1 - y_{weighted, n}) \log (1 - \sigma(x_{weighted, n}^\top \btheta))
\big).
\end{aligned}
\end{equation}

\subsection{Applying AMIP to BT models in logistic form}
\label{sec:amip-approximation}
In this section we provide details on applying general \cref{eq:derivative-expanded} in our specific case of logistic regression formed BT models. We observed that our quantity of interest $\theta_i-\theta_j$ is a linear combination of effect size $\theta_i$s in logistic regression, thus the first order Taylor expansion of this quantity can be calculated by first order Taylor expansion of $\theta_i$s. 

Let $e_j$ denote the $j$th standard basis vector and $\designmatrix \in \mathbb{R}^{N \times P}$ denote the design matrix. Let \(\widehat{p}_n = \sigma( \hat{\theta}^\top x_n)\) and $\bV = \operatorname{diag}(\{\widehat{p}_n (1 - \widehat{p}_n)\}_n)$. For logistic regression with an effect-size quantity of interest, $\theta_j$, the formula for the influence score for the $n$th data point \citep{pregibon1981logdiagnostics} is given by
\begin{align}
    \frac{\partial \hat{\theta}_j(\wvec)}{\partial w_n}\Big\vert_{\wvec = 1_N} = 
    e_j^\top(\designmatrix^\top \bV \designmatrix)^{-1}x_n\widehat{p}_n (1 - \widehat{p}_n)\left(y_n - \widehat{p}_n\right),
    \label{eqn:logreg-influence-function-app}
\end{align}
In addition to influence functions, our framework enables a second data-dropping approximation known as the \textit{One-step Newton (1sN)} approximation, which approximates the effect of dropping a data point on model parameters using a second-order Taylor expansion in parameter space. This Newton-style update has become popular for approximating the deletion of data in recent works on approximate cross validation \citep{ghosh2020approximate,wilson2020approximate} and machine unlearning \citep{sekhari2021remember,suriyakumar2022algorithms}. The 1sN is slightly more expensive to compute than the IF approximation (as it corrects the IF with a multiplicative correction term) but is more accurate when the to-be-dropped data point has high a leverage score (because the correction term involves the leverage score of a data point). Previous works have proposed approximating the removal of a group of data points by the sum of leave-one-out 1sN scores, in an algorithm known as the \textbf{Additive one-step Newton approximation} \citep{huang2024approximations,park2023trak}.

\revision{To run the AMIP and Additive one-step Newton algorithm to check pairwise robustness between two given players, $i$ and $j$, we: }
\begin{enumerate}
    \item Fit a BT model on the entire arena.
    \item Compute the \emph{influence scores} (\cref{eqn:logreg-influence-function-app}) (one-step Newton scores for the Additive one-step Newton algorithm) for all matches in the arena.
    \item Identify the $\floor{\alpha N}$ matchups for which the difference in influence scores is the largest in the negative direction (assuming that player $i$ has a higher estimated BT score than player $j$ on the full data).
    \item \revision{Approximate impact of dropping these $\floor{\alpha N}$ matchups by the sum of the influence score approximations.}
    \item If the approximation predicts that the relative ranking between players $i$ and $j$ changed, then refit the model leaving out the identified subgroup.\footnote{Our algorithm gives users the option to refit the BT model for all matchups, regardless of whether a predicted ranking change occurs.}
\end{enumerate}
These data-dropping algorithms replace a computationally intractable combinatorial search with an algorithm that costs only \[O(Analysis + N \log(\alpha N) + NP^2 + P^3),\] where $Analysis$ represents the cost of fitting the initial Bradley--Terry model on the original arena to compute scores. Data-dropping approximations make identifying candidate subsets of the arena that may induce top-$\topk$ non-robustness very fast because they eliminate the need to retrain the BT model repeatedly on every candidate subset. Once a candidate subset is identified, however, our method always performs a \textit{refitting} of the BT model with the identified subset removed to verify whether the non-robustness is true. This final verification step ensures that our method does not return false positives.

\section{Arenas}
\label{sec:arenas}
\textbf{Chatbot Arena.} A crowdsourced platform where users engage in conversations with two chatbots at the same time and rate their responses based on personal preferences \citep{zheng2023judging}. We use the \texttt{arena-human-preference-55k} amd \texttt{chatbot-arena-llm-judges} datasets. This benchmark contains a total number of \(57{,}477\) preferences. \rev{Figure~\ref{figure:confint-cba-human}} presents the BT scores of the top models in Chatbot Arena.
% \begin{figure}[t]
%     \centering
%     \includegraphics[width=0.85\linewidth]{figures/top10_cba.png}
%     \vspace{-0.42cm}
%     \caption{The top 10 models according to Chatbot Arena rankings in 2024.}
%     \vspace{-0.6cm}
%     \label{fig:top10-chatbot-arena}
% \end{figure}

\textbf{MT-Bench.} A multi-turn question set designed to compare LLMs in multi-turn conversation and instruction following constructed to distinguish between models based on reasoning and mathematics \citep{zheng2023judging}. We use the \texttt{mt-bench-human-judgments} dataset. This benchmark was handcrafted using $58$ expert-level human labelers; it contains $3{,}355$ total preferences. In contrast to Chatbot Arena, labelers are mostly graduate students, so they are considered more skilled than average crowd workers. \rev{Figure~\ref{figure:confint-MTBench-human}} presents the BT scores of the models in MT-bench.
% \begin{figure}[t]
%     \centering
%     \includegraphics[width=0.65\linewidth]{figures/top6_mtb_human.png}
%     \caption{The model rankings on MT-Bench.}
%     \label{fig:top6-mt-bench}
% \end{figure}

\textbf{Search Arena.} A crowdsourced platform for search-augmented LLMs, focusing on real-world and current events rather than static factual questions. We conduct our analysis using historical data available on Hugging Face: \texttt{lmarena-ai/search-arena-24k}. The dataset contains $24{,}069$ multi-turn conversations with search-LLMs across diverse intents, languages, and topics. % https://news.lmarena.ai/search-arena/
\rev{Figure~\ref{figure:confint-search}} presents the BT scores of the top models in Search Arena.
% \begin{figure}[t]
%     \centering
%     \includegraphics[width=0.65\linewidth]{figures/top_search_arena.png}
%     \caption{The model rankings on Search Arena.}
%     \label{fig:top_search}
% \end{figure}

\textbf{Webdev Arena.} A crowdsourced platform for LLM web development tasks, such as building interactive applications and webpages. We conduct our analysis using historical data available on Hugging Face: \texttt{lmarena-ai/webdev-arena-preference-10k}. This dataset contains $10{,}000$ user-submitted prompts.  % https://news.lmarena.ai/webdev-arena/
\rev{Figure~\ref{figure:confint-webdev}} presents the BT scores of the top models in Webdev Arena.
% \begin{figure}[t]
%     \centering
%     \includegraphics[width=0.65\linewidth]{figures/top10_webdev.png}
%     \caption{The top-10 model rankings on Webdev Arena.}
%     \label{fig:top_webdev}
% \end{figure}

\textbf{Vision Arena.} A crowdsourced platform that tests vision-language models on visual question-answering. There are a total of $30{,}000$ single and multi-turn chats between users and two anonymous vision-language models. We conduct our analysis using historical data available on Hugging Face: \texttt{lmarena-ai/VisionArena-Battle}.  % https://huggingface.co/datasets/lmarena-ai/VisionArena-Battle
\rev{Figure~\ref{figure:confint-vision}} presents the BT scores of the top models in Vision Arena.
% \begin{figure}[t]
%     \centering
%     \includegraphics[width=0.65\linewidth]{figures/top10_vision.png}
%     \caption{The top-10 model rankings on Vision Arena.}
%     \label{fig:top_vision}
% \end{figure}

\textbf{ATP Tennis.} Association of Tennis Professionals (ATP) tennis records consolidated by \citet{Sackmann2024TennisATP}. Each entry represents a match from the ATP tour, a worldwide top-tier men's tennis tour, and consists of the identifiers of the winning and losing players and the match-related metadata (e.g., player rankings, name of the tournament). We focused on the top-10 ranked players based on the 2024 season ranking and analyzed their plays throughout four seasons, 2020-2024. To avoid the case where dropping a small proportion of matches could drop a player's entire record, we focus our analysis on players who played at least 20 games. There were in total 278 games after filtering. Figure~\ref{fig:top_tennis} presents the BT scores of the top models in the tennis dataset.
\begin{figure}[t]
    \centering
    \includegraphics[width=0.65\linewidth]{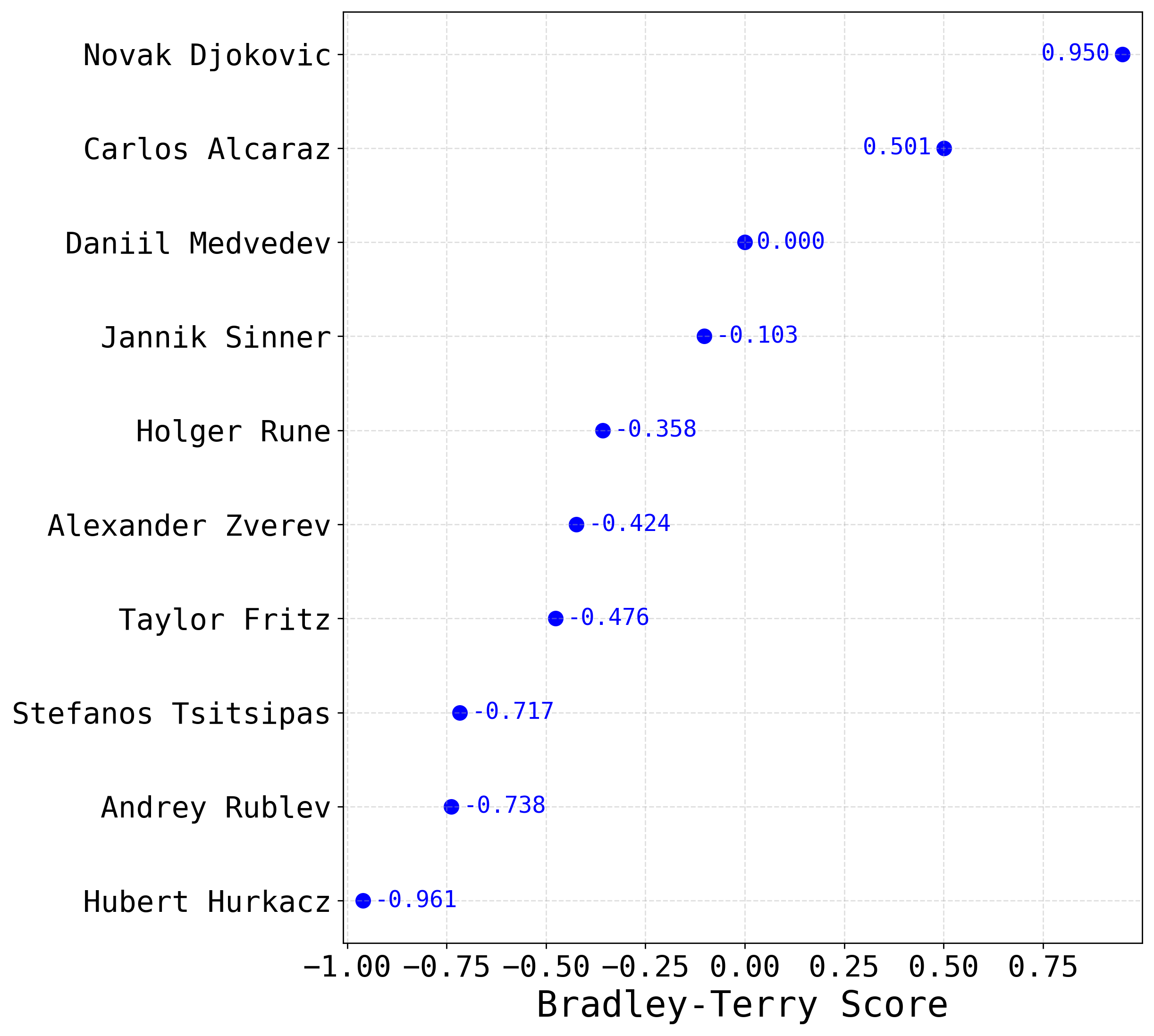}
    \caption{The top-10 player rankings in the tennis data.}
    \label{fig:top_tennis}
\end{figure}

\textbf{NBA.} Basketball games from all seasons of the National Basketball Association (NBA), consolidated by \citet{nbaelofivethirtyeight}. Each entry represents a historical game from the National Basketball Association, consisting of the identifiers of the two teams, the outcome of the game (win or loss), as well as game-related metadata (e.g., Elo score of each team, game location). To avoid the case where dropping a small proportion of matches could drop a player's entire record, we focus our analysis on the top 50 teams by number of games played. There are a total of 109,892 matchups between the 50 teams. Figure~\ref{fig:top20-nba} presents the BT scores of the top teams in the NBA.
\begin{figure}[t]
    \centering
    \includegraphics[width=0.7\linewidth]{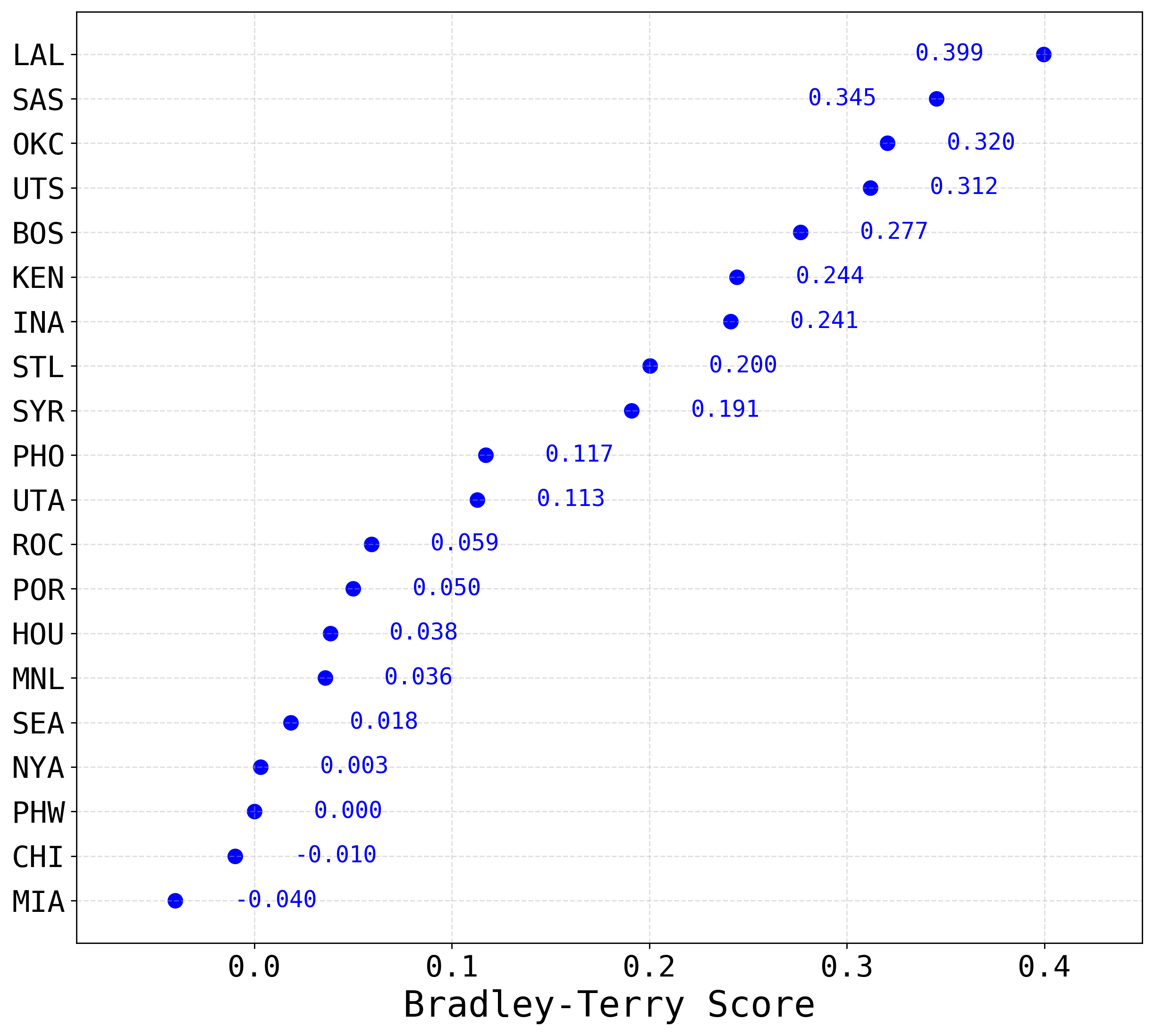}
    \caption{The top-20 team rankings in the NBA.}
    \label{fig:top20-nba}
\end{figure}

\section{\revision{Player Involvement, Homogeneous Bars}}
\label{sec:player-involvement-supplementals}
\revision{Across all top-$\topk$ robustness experiments, $100\%$ of dropped matches involved either one or both of the models whose rankings were flipped, with $100\%$ belonging to one of these two cases within a given $\topk$ (see \cref{fig:player_involvement_comparison}). There are no partial bars or mixed compositions. Readers may ask: Why does this homogeneous pattern consistently appear? Could this be a property of the arena data?}

\revision{We investigate this by manually inspecting the dropped matchups returned by our robustness assessing algorithm for each value of $\topk$. Specifically, in each case, we identified the dropped matchups and inspected which players appeared in these matchups. We summarize the findings here:}
\begin{itemize}
    \item \revision{$\topk=1$: 2 games were dropped to flip GPT-4-0125-preview (originally 1st) and GPT-4-1106-preview (2nd). These two matches were between GPT-4-0125-preview and two other models, vicuna-13b (22nd) and stripedhyena-nous-7b (45th), with GPT-4-0125-preview losing.}
    \item \revision{$\topk=3$: 29 games were dropped to flip models gpt-4-0314 (3rd place) with mistral-7b-instruct-v0.2 (6th place). Games were played between mistral-7b-instruct-v0.2 and various other models, with mistral-7b-instruct-v0.2 losing all matches.}
    \item \revision{$\topk=5$: 3 games were dropped to flip models qwen1.5-72b-chat (5th place) with mistral-medium (6th place). All dropped matches were between qwen1.5-72b-chat and gpt-4-1106-preview (1st place), with qwen1.5-72b-chat (5th place) winning.}
    \item \revision{$\topk=10$: 1 game was dropped to flip models gemini-pro (10th) and mixtral-8x7b-instruct-v0.1 (11th place). The dropped match was between the two models, with gemini-pro winning.}
    \item \revision{$\topk=20$: 1 game was dropped to flip models gpt-3.5-turbo-0314 (20th place) with nous-hermes-2-mixtral-8x7b-dpo (21st place). The dropped match was between nous-hermes-2-mixtral-8x7b-dpo (21st place) and vicuna-13b (22st place), with nous-hermes-2-mixtral-8x7b-dpo losing.}
\end{itemize}
\revision{The reason the involvement is always entirely either one or both affected players is because all of the dropped matchups consist of games played between a central model and a specific competitor (or group of competitors) whose outcomes all favor or disfavor the specific model \rev{and every dropped preference was a clear win or loss (no ties), aligning in the direction required to flip the ranking. In other words, whenever a top-$\topk$ set changed due to the demotion of a model, all dropped matches were ones that the demoted model had originally won, and vice versa for promotions.} This structure then leads the dropped matchups to consist entirely of evaluations that involved one or both ranking-flipped models. This finding reveals something interesting about the nature of the non-robustness in our analysis: small, consistent sets of matchups are sufficient to push a model just above or below another on the leaderboard.}

%\subsection{Player Involvement in Dropped Matches} 

For every instance where the top-$\topk$ leaderboard changes due to dropped preferences, we find that the affected matches always involve at least one of the models whose rank is altered (see \cref{fig:player_involvement_comparison}). This holds true for both human-judged and LLM-judged Chatbot Arenas. While \citet{min2025improving} find that adding in a small fraction of rigged votes can influence a target model’s ranking even when the target model is not directly involved in the rigged votes, we are unable to find instances where rankings were flipped by removing a small fraction of preferences where neither of the affected models were involved.

Also, notice in \cref{fig:player_involvement_comparison} that there are no partial bars or mixed compositions. We investigate why this homogeneous pattern appears consistently across bars. Inspecting dropped matchups manually, we find that the reason why one or both flipped players are always involved in the dropped matchups is because these matchups are always played between the model that is flipped, call it the target model, and a specific competitor (either the model whose ranking is flipped relative to the target model, or another model) or group of competitors (including models whose rankings remain unchanged), and all matchups either always favor or disfavor the target model (see \cref{sec:player-involvement-supplementals} for a more detailed description).
%specific competitor (either the model whose ranking is flipped with this model or another model) or a group of models (including models whose rankings are not flipped), whose outcomes all favor or disfavor that specific model, a structure that leads to homogeneous bars. 
This finding reveals something about how non-robustness appears in our analyses: small, consistent sets of matchups are sufficient to push a model just above or below another on the leaderboard.

% Thus, we were unable to find evidence here that omniscience (i.e., where someone might strategically remove matches unrelated to the target players in order to alter the rankings \citet{min2025improving}) helps with vote-rigging in Chatbot Arena, as found in \citet{min2025improving}.
\begin{figure}[!ht]
    \centering

    \begin{subfigure}[t]{0.49\textwidth}
        \centering
        \includegraphics[width=\linewidth]{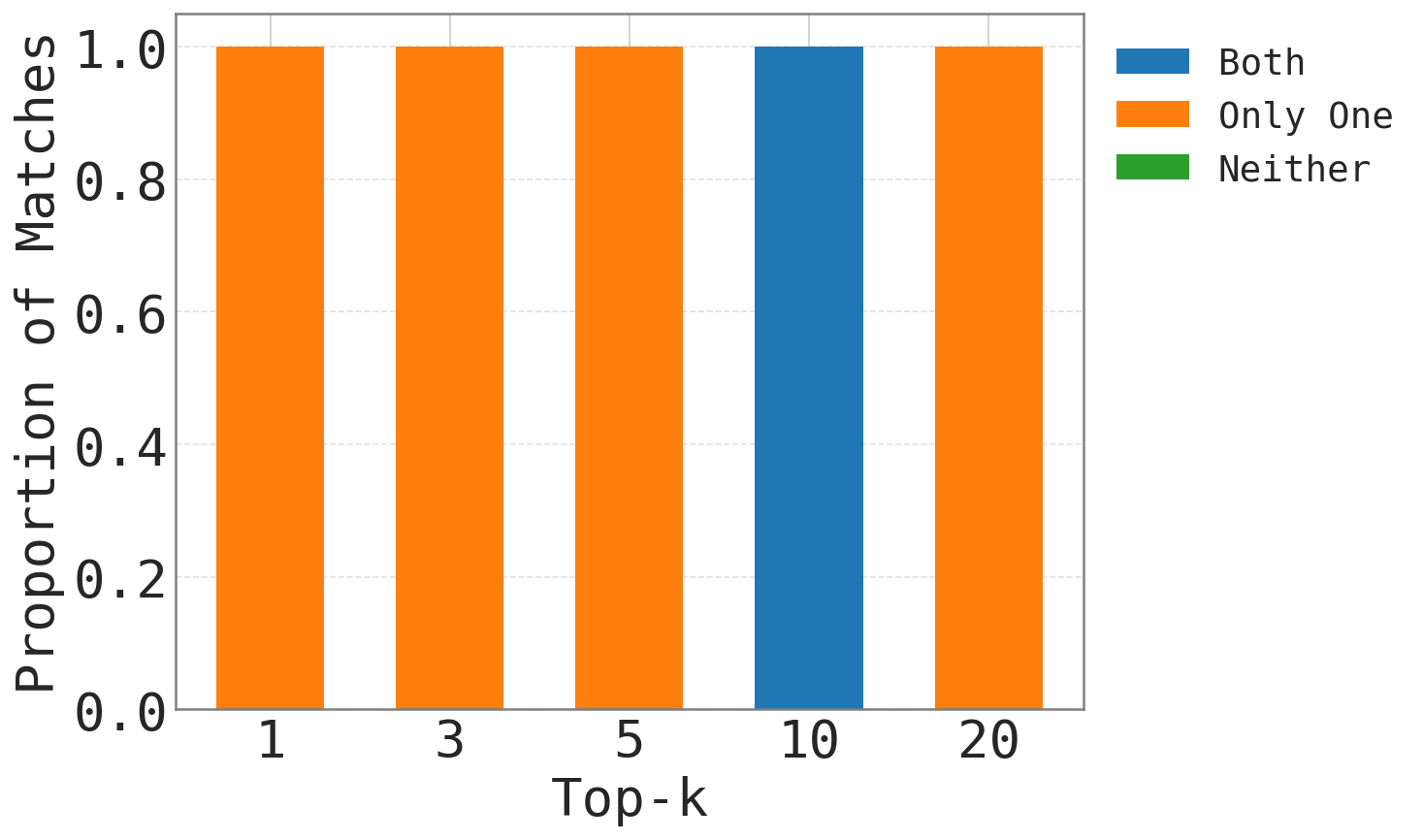}
        \caption{Chatbot Arena (Human-Judge)}
        \label{fig:chatbot_arena}
    \end{subfigure}
    \hfill
    \begin{subfigure}[t]{0.49\textwidth}
        \centering
        \includegraphics[width=\linewidth]{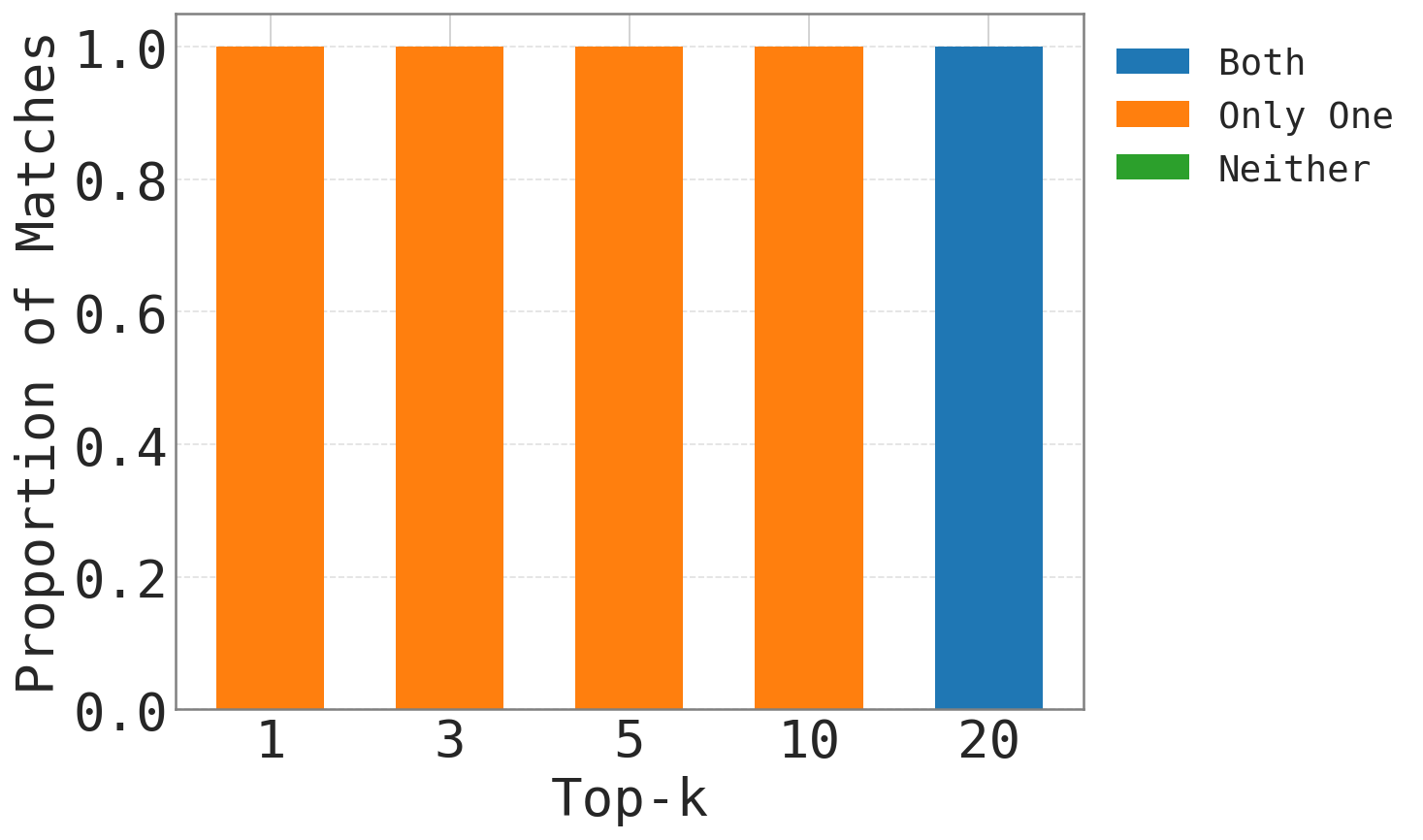}
        \caption{Chatbot Arena (LLM-Judge)}
        \label{fig:llm_arena}
    \end{subfigure}

    \caption{\textit{Player involvement} in the most influential matches whose removal caused two models (players), one inside the top-$\topk$ positions and one outside, to exchange places. Each bar represents the composition of matches dropped in order to exchange the models. The proportions indicate whether neither (green), one (orange), or both (blue) of the affected models were involved in each dropped matchup. For Chatbot Arena (Human-Judge) (top), the number of matches dropped for each $k \in \{1, 3, 5, 10, 20\}$ is $(2,29,3,1,1)$, respectively. For Chatbot Arena (LLM-Judge) (bottom), the number of matches dropped for each $k \in \{1, 3, 5, 10, 20\}$ is $(9,24,9,1,2)$. Across all $k$, we find that $100\%$ of flipped rankings involved matches containing at least one of the players whose ranking changed.}
    \label{fig:player_involvement_comparison}
\end{figure}

\section{Inspecting Dropped Preferences}
\label{sec:dropped-preferences}
%\FloatBarrier
Dropping just two (\(0.003\%\) of) preferences on Chatbot Arena is enough to change the top-ranked model on Chatbot Arena from GPT-4-1106-preview to GPT-4-0125-preview. Below we provide the two prompts and response pairs responsible for the ranking flip, together with the corresponding annotation. \rev{Additionally, we note that the BT model’s estimation procedure does not use any information about prompt content; it depends only on the BT scores of the two models involved in each match, which means it is possible for dropped subsets to be non-unique.}

\paragraph{Prompt 1.} ``teach me how to make independent classes in python that can be call and install through pip.''

\paragraph{Competing Models.} GPT-4-0125-preview, Vicuna-13b

\paragraph{Preferred Model.} \textit{Vicuna-13b.}

\paragraph{Responses.} See \cref{fig:prompt-1-gpt} for the complete response by GPT-4-0125-preview and \cref{fig:prompt-1-vicuna} for the complete response by Vicuna-13b.

% \begin{figure}[b]
%     \centering
%     \includegraphics[width=0.65\linewidth]{figures/prompt-1-gpt.pdf}
%     \caption{GPT-4-0125-preview: response to prompt 1.}
%     \label{fig:prompt-1-gpt}

%     \vspace{0.5em} % small spacing between figures

%     \includegraphics[width=0.65\linewidth]{figures/prompt-1-vicuna.pdf}
%     \caption{Vicuna-13b: response to prompt 1.}
%     \label{fig:prompt-1-vicuna}
% \end{figure}

\begin{figure}[b]
    \centering
    \includegraphics[width=0.85\linewidth]{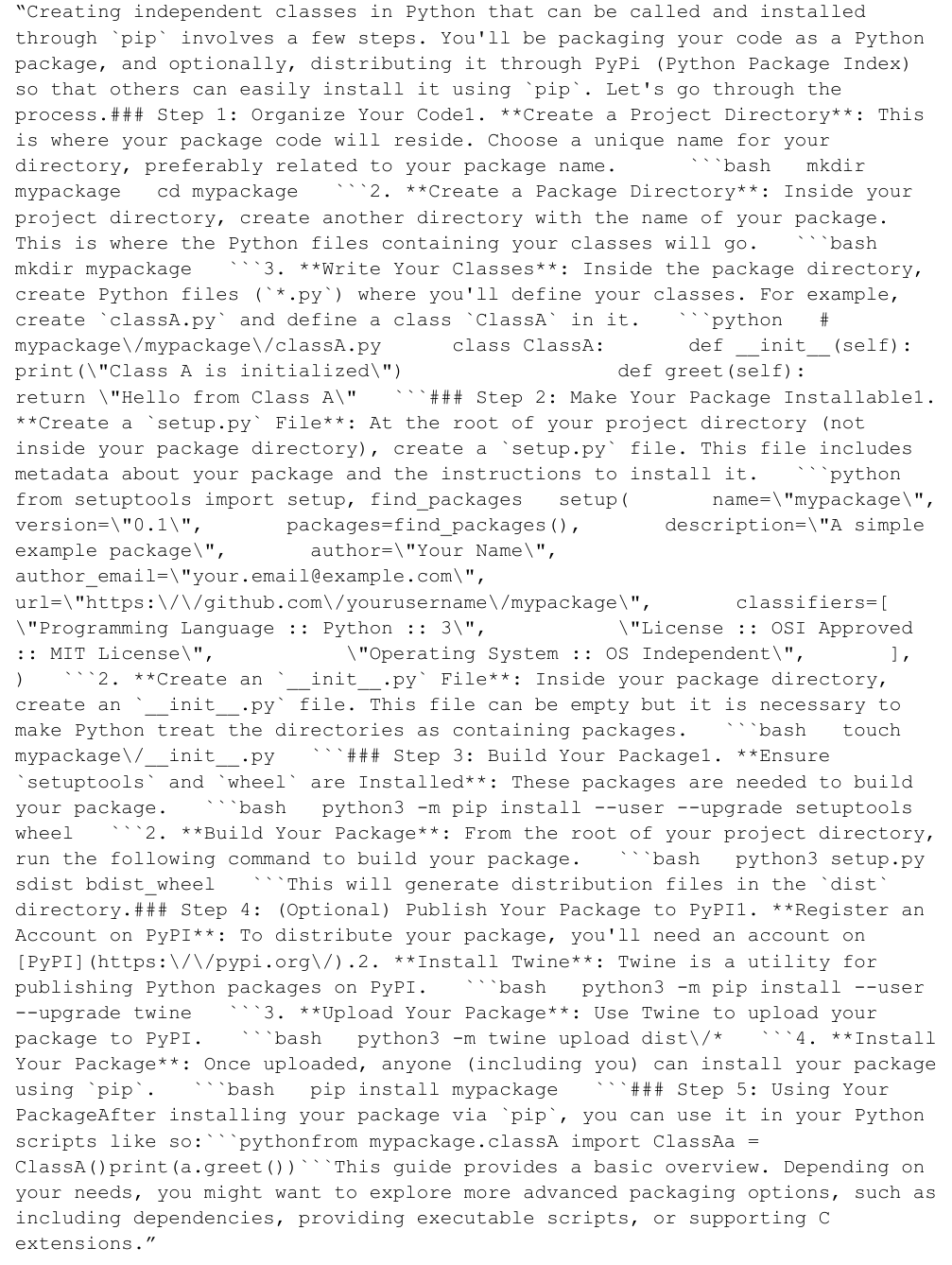}
    \caption{GPT-4-0125-preview: response to prompt 1.}
    \label{fig:prompt-1-gpt}
\end{figure}

\begin{figure}[b]
    \centering
    \includegraphics[width=0.85\linewidth]{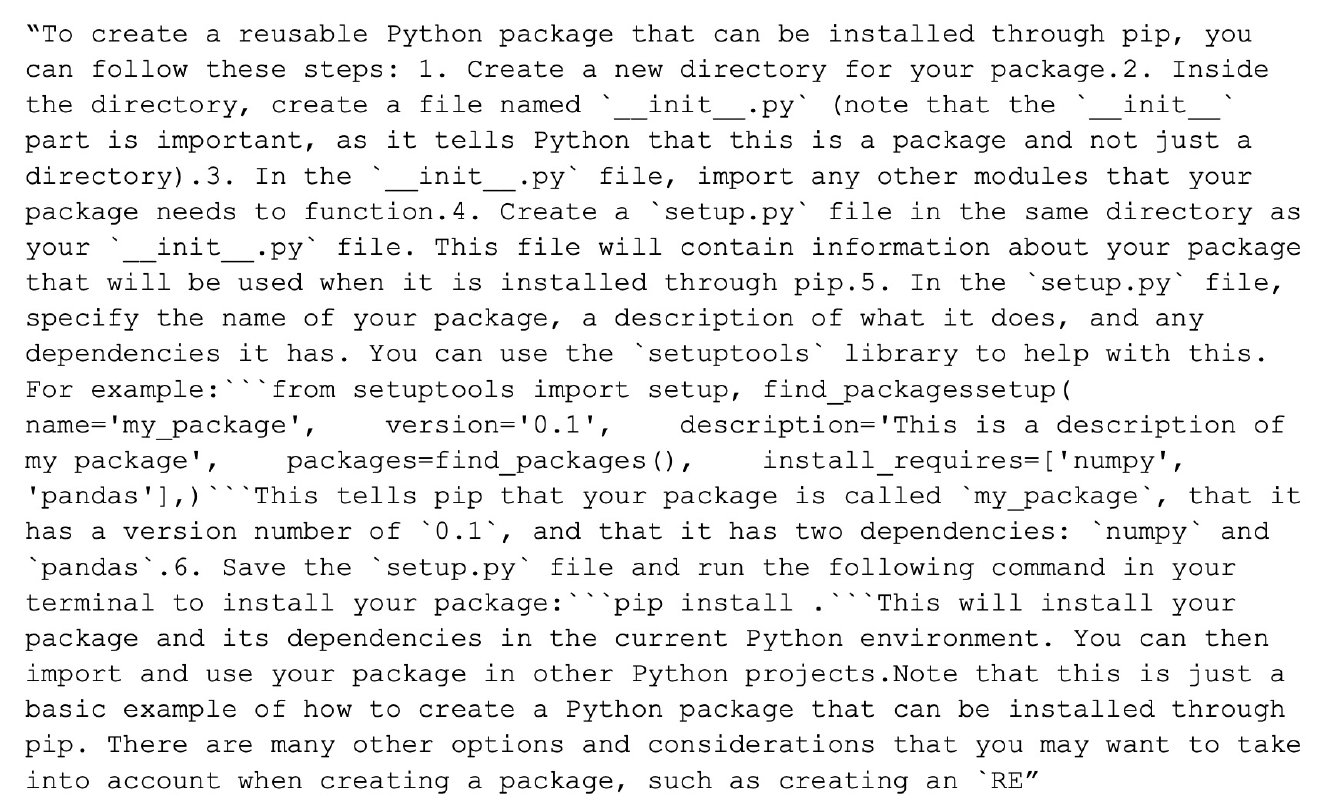}
    \caption{Vicuna-13b: response to prompt 1.}
    \label{fig:prompt-1-vicuna}
\end{figure}

\paragraph{Prompt 2.} ``name me challenging c++ projects i can add on my cv as a cs student. give me creative ideas and dont forget im a student.''

\paragraph{Competing Models.} GPT-4-0125-preview, Stripedhyena-nous-7b

\paragraph{Preferred Model.} \textit{Stripedhyena-nous-7b.}

\paragraph{Responses.} See \cref{fig:prompt-2-gpt} for the complete response by GPT-4-0125-preview and \cref{fig:prompt-2-hyena} for the complete response by Stripedhyena-nous-7b.

% \begin{figure}[b]
%     \centering
%     \includegraphics[width=0.7\linewidth]{figures/prompt-2-gpt.pdf}
%     \caption{GPT-4-0125-preview: response to prompt 2.}
%     \label{fig:prompt-2-gpt}

%     \vspace{0.5em}

%     \includegraphics[width=0.7\linewidth]{figures/prompt-2-hyena.pdf}
%     \caption{Stripedhyena-nous-7b: response to prompt 2.}
%     \label{fig:prompt-2-hyena}
% \end{figure}

\begin{figure}[b]
    \centering
    \includegraphics[width=0.85\linewidth]{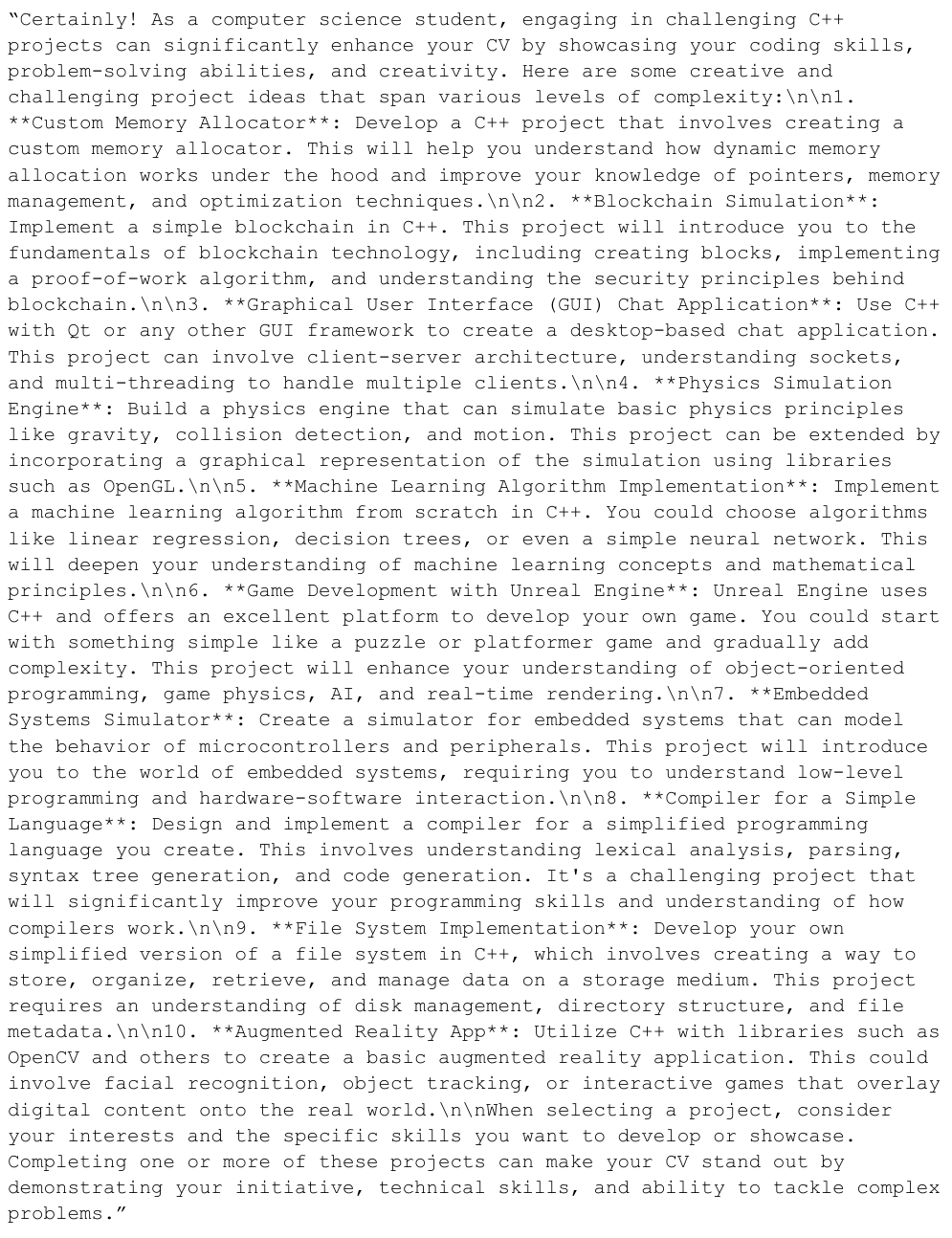}
    \caption{GPT-4-0125-preview: response to prompt 2.}
    \label{fig:prompt-2-gpt}
\end{figure}

\begin{figure}[b]
    \centering
    \includegraphics[width=0.85\linewidth]{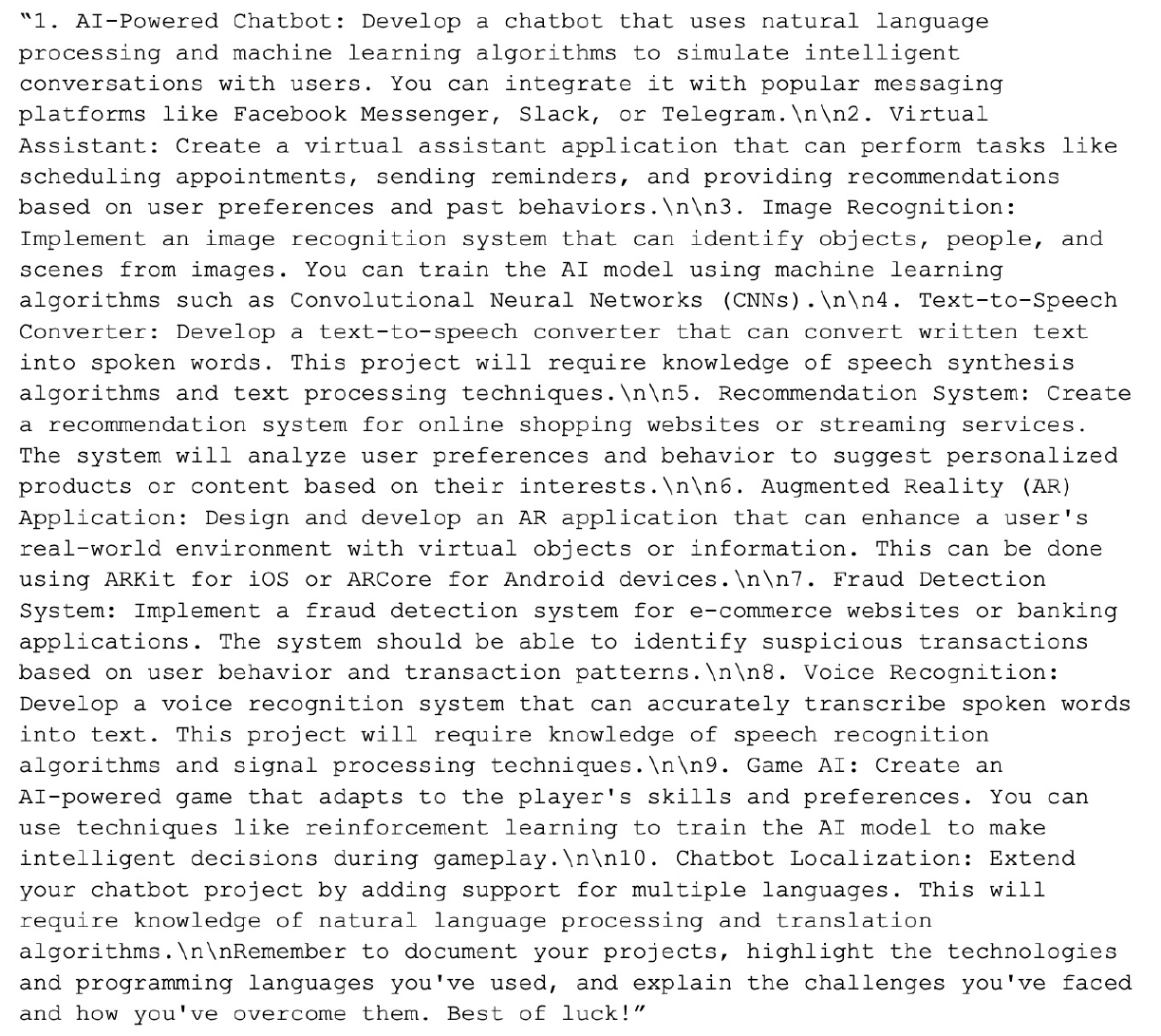}
    \caption{Stripedhyena-nous-7b: response to prompt 2.}
    \label{fig:prompt-2-hyena}
\end{figure}

\rev{To conduct a qualitative analysis on the response pairs, we provide a strong judge model, GPT-5.1\footnote{\rev{We accessed GPT-5.1 through the ChatGPT user interface in a logged-out, incognito session.}}, with a system prompt requesting a summary of the two responses, their similarities and differences, and a difficulty rating indicating how distinct the responses are.}

\paragraph{\rev{System Prompt}}
\begin{quote}
\rev{Provide a summary of the following two AI-generated responses to the prompt: [Insert Prompt Here]. Then describe their main similarities and differences. Finally, rate on a scale from 1 (very similar) to 5 (one response clearly better) how difficult it would be for a user to distinguish the two responses.}

\vspace{0.5em}
\rev{\textbf{Response A:} [Insert Response A here]}

\rev{\textbf{Response B:} [Insert Response B here]}
\vspace{0.5em}

\rev{Provide 2-3 sentences for each of the following headings: summary, similarities, differences, and difficulty rating.}
\end{quote}

\rev{In the following GPT-5.1 responses, Response A refers to GPT-4-0125-preview, and Response B refers to the smaller, open-source models, Vicuna-13b and Stripedhyena-nous-7b, respectively for prompts 1 and 2.}

\paragraph{\rev{GPT-5.1-Generated Response [for Prompt 1]}}

\begin{quote}
\rev{\paragraph{Summary.} Response A provides a detailed, step-by-step explanation of how to create Python classes, organize them into a package, build distribution files, and optionally upload the package to PyPI. It includes full code examples, terminal commands, and practical guidance for beginners. Response B offers a brief overview of how to create a package with init.py and setup.py, focusing mainly on local installation and basic configuration, and ends abruptly without deeper detail.}
\rev{\paragraph{Similarities.} Both responses explain that Python packages require an init.py file and a setup.py file using setuptools. They also describe how a user can install the finished package with pip. Each provides an example of what a setup.py file might look like.}
\rev{\paragraph{Differences.} Response A is much more comprehensive, guiding the reader through class creation, directory structure, building distribution files, and publishing to PyPI, while Response B offers only a high-level outline of the process. Response A includes practical, runnable examples, whereas Response B only sketches the idea of the setup. Response B also feels incomplete and omits several important steps that Response A covers thoroughly.}
\rev{\paragraph{Difficulty Rating.} On a scale of 1 to 5, where 1 means the responses are very similar and 5 means one is clearly better, the two responses would rate around a 4.5. A typical user would likely find Response A significantly more helpful, complete, and actionable than Response B.}
\end{quote}

\paragraph{\rev{GPT-5.1-Generated Response [for Prompt 2]}}

\begin{quote}
\rev{\paragraph{Summary.} Response A focuses on C++-specific, systems-oriented projects such as memory allocators, compilers, and physics engines, emphasizing core CS concepts and low-level programming. Response B centers on AI- and application-focused ideas like chatbots, image recognition, and recommendation systems, many of which are not inherently tied to C++.}
\rev{\paragraph{Similarities.} Both responses offer lists of challenging project ideas intended to help a CS student strengthen their CV. They also highlight projects that can demonstrate technical skills and encourage documenting the work effectively.}
\rev{\paragraph{Differences.} Response A is tightly aligned with the prompt, emphasizing C++ and foundational CS mechanics, whereas Response B presents more general software and AI projects that may rely on other languages or frameworks. Additionally, Response A provides deeper technical context, while Response B takes a more high-level, application-oriented approach.}
\rev{\paragraph{Difficulty Rating.} Rating: 5 — The two responses differ clearly in focus, depth, and C++ relevance, making them easy to distinguish. A typical user would quickly notice that only Response A directly addresses the C++-project requirement.}
\end{quote}

\rev{In both surfaced examples, GPT-5.1 judges the pair of responses to be easy to differentiate and consistently prefers the opposite response from the human annotator (e.g., ``A typical user would likely find Response A significantly more helpful, complete, and actionable than Response B,'' and ``A typical user would quickly notice that only Response A directly addresses the C++-project requirement.''). This makes sense, as both matches are cases in which a much lower-scoring model is preferred to the top-ranked model. Thus, one might interpret the influential subsets the method identifies as ``outlier’’ preferences, cases where the annotator’s preference deviates from what the average user might select.}

\subsection{Sensitivity Driven by Narrow Score Margins}\label{sec:experiments:score_margins}

We find that the stability of the arena depends on the BT score margins between models (see \cref{fig:robustness-vs-score-gap}). Recall from \cref{table:nonrobustness-table} that %, leaderboards are incredibly fragile: 
dropping only two preferences is enough to change the top-ranked model. To explore the effect of score margins, we first remove all games involving the second-place model (GPT-4-1106-preview). The arena then becomes more resilient, requiring dropping 38 out of $57{,}477$ ($0.07\%$) preferences to overturn the leader. When we further remove all games involving the 2nd through 5th place models, the leaderboard becomes harder to perturb, but is still remarkably sensitive, requiring dropping 63 out of $57{,}477$ ($0.1\%$) preferences to flip the top model.
% question to investigate: what proportion of the top 2 through k players would we have to remove to make Chatbot Arena robust (at alpha = 1%)?

% Fragility arises when top competitors are indistinguishable from a preference standpoint--annotators cannot reliably separate their performance based on the prompts submitted to the platform. In such settings, even a handful of preferences can change the leaderboard. Once these ``close rivals" are excluded, however, the arena becomes more robustness.

% To mitigate sensitivity then, preference data should be collected in ways that sharpen distinctions between closely-matched models. Evidence from MT-bench suggests that robustness improves with expert annotators and specially curated prompts targeting difficult domains like mathematics, coding, and multi-turn reasoning \citep{zheng2023judging}. These design choices increase signal, leading to more robust model rankings.

One possible explanation for this fragility is that top competitors are often closely matched, making it difficult for annotators to reliably separate their performance on the prompts submitted to the arena. This raises the possibility that sensitivity could be reduced by that sharpens distinctions between models (for example, through expert annotators and curated prompts targeting challenging domains such as mathematics, coding, and multi-turn reasoning, as in MT-Bench \citep{zheng2023judging}).

\section{\rev{Verification of AMIP-identified Subsets}}
\label{sec:w-vec-verification}

\rev{For each of our nine data analyses, we can use the machinery of our method to return a $\wvec$ corresponding to a smallest data subset that can be dropped to change the top-1 ranking. The machinery of our method also returns an estimate (before re-running the BT model) for whether the top-1 ranking is changed. To examine how often an identified weight vector $\wvec$ truly corresponds to a subset whose removal flips the ranking, we report across all nine arenas the number of cases where the estimate with the $\wvec$-vector accurately reflects a change in the top ranking over the total number of arenas tested for top-1 robustness.}

\rev{In \cref{table:verify-amip-subsets}, we find that all identified $\wvec$-vectors lead to a true change in ranking. And we find that this result holds even when the dropped subset is greater than $\floor{\alpha N}$ of the data (even though the original AMIP makes no claims to an accurate identification of a decision-changing $\wvec$ in this case). However, this does not mean that there are no cases where AMIP fails to surface a vector $\wvec$ that leads to a change in ranking (i.e., false negatives are possible).}

% \rev{For each dataset, we run AMIP to obtain the most influential subset, drop these matchups, and recompute BT scores exactly. We then record whether the ranking flip predicted by the approximation does occur. All nine arenas exhibit a true ranking flip when their AMIP-identified subset is removed (see \cref{table:verify-amip-subsets}). This result shows that all surfaced $w$-vectors in our experiments correspond to valid identifications of non-robustness.}
% \rev{We note, however, that in the cases of MT-Bench and ATP Tennis, we also return the subsets  AMIP identifies, even when they pass the $1\%$ threshold (the $\alpha$-level we set in our experiments), a step that is not part of the original AMIP algorithm \citep{broderick2023automatic}. In fact, AMIP, an influence-functions based approximation, may not be expected to perform well for approximating the removal of more than a small fraction of points (e.g., $1\%$ of the sample).}

\begin{table}[ht]
\centering
\small
\begin{tabular}{p{3.5cm} p{8.5cm} c}
\toprule
\textbf{Dataset} & \textbf{AMIP-Returned Subset (Indices)} & \textbf{Flip?} \\
\midrule

Chatbot Arena &
\{46592, 5156\} &
Yes \\

Vision Arena &
\{22176, 9686, 887, 15782, 24340, 25110, 9816, 10926, 18732, 21303, 13957, 2934, 2936, 19600, 11072, 15311, 11038, 25845, 17732, 29100, 5421, 24462, 23006, 10572, 2134, 13518, 5390, 15353\} &
Yes \\

NBA Games &
\{18819, 19717, 18818, 19762, 14523, 19763, 14522, 20900, 22132, 22133, 18305, 15756, 14383, 18304, 14382, 19716, 20135\} &
Yes \\

Chatbot Arena (LLM-judge) &
\{41445, 9108, 14834, 11144, 11675, 9123, 17291, 48894, 42411\} &
Yes \\

Webdev Arena &
\{7164, 7539, 9112, 7711, 2089, 1815, 2414, 6542, 6446, 4883, 8753, 2889, 9272, 3553, 1512, 5933, 6992, 10387\} &
Yes \\

Search Arena &
\{22164, 12847, 12819, 21810, 11852, 19956, 9492, 15447, 11324, 16583, 12733, 10116, 21940, 15552, 9451, 12602, 21977, 11499, 12576, 10146, 12557, 11519, 15699, 9420, 12851, 18068, 12931, 11278, 13279, 11143, 11163, 21587, 9963, 13226, 9586, 20632, 13191, 9978, 13189, 12456, 11204, 17160, 13129, 18238, 18231, 10009, 13112, 15234, 11251, 20575, 13043, 10030, 11209, 9607, 20336, 15733, 22646, 12061, 11768, 12023, 10375\} &
Yes \\

MT-bench (LLM-judge) &
\{646, 587, 1290, 1741, 720, 570, 571, 72, 223, 1212, 1183, 1122, 2052, 2053, 2112, 1242, 1063, 1033, 1032, 1003, 1812, 2113, 1002, 282, 1093, 1092, 1243, 2022, 1753, 1752, 132, 103, 102, 1872, 1873, 1543, 162, 1453, 1423, 1422\} &
Yes \\

ATP Tennis &
\{236, 168, 251, 177, 202, 122\} &
Yes \\

MT-bench (Human-judge) &
\{137, 2399, 1298, 1884, 2398, 139, 1153, 850, 391, 1111, 3181, 91, 648, 2612, 803, 802, 804, 801, 800, 348, 744, 41, 2726, 349, 2668, 608, 607, 1450, 799, 2909, 1409, 2912, 2725, 748, 2492, 1537, 160, 1536, 2911, 1534, 925, 1535, 2333, 2161, 570, 1830, 346, 2334, 745, 1408, 1191, 2332, 3055, 101, 222, 2883, 3274, 221, 2837, 219, 667, 178, 3021, 3022, 1902, 2552, 2551, 2341, 863, 1124, 1903, 2624, 2626, 2627, 1634, 898, 1744, 2510, 1745, 220, 3275, 666, 1162, 246, 1214, 1294, 1165, 64, 247, 1556, 65, 3278\} &
Yes \\

\bottomrule
\end{tabular}
\caption{\rev{For each dataset, the number of cases where the estimate with the $\wvec$-vector accurately reflects a change in the top ranking) over the total number of arenas tested for top-1 robustness. All surfaced $\wvec$-vectors successfully flip the ranking (9/9).}}
\label{table:verify-amip-subsets}
\end{table}

\section{\rev{Masking Effects and the Possibility of False Negatives}}
\label{sec:false-negatives}

\rev{A main limitation of our approach is that while it can conclusively identify non-robustness, it is possible that there is non-robustness that it does not find: when our method surfaces a subset whose removal flips the ranking, the resulting perturbation is an exact, verifiable witness of fragility; however, when no such subset is found, we cannot conclude that the arena is robust.}

\rev{This limitation is a documented challenge in the literature on identifying influential subsets \citep{hu2024most,huang2024approximations,moitra2022provably}. In linear models, for example, \citet{huang2024approximations} and \citet{hu2024most} show that AMIP and related additive, first-order approximations can miss influential subsets. A key failure mode is due to a phenomenon is known as ``masking,'' in which several highly-impactful data points produce a large change to a statistic (e.g., an estimated BT-score) when deleted jointly, yet no single point appears influential when considered in isolation. To address masking effects, works such as \citet{belsley2005regression,kuschnig2021hidden,huang2024approximations} have considered using step-wise (greedy) approaches, of removing the most influential points in sequence. The main empirical conclusion of this paper relies on an existence proof: namely, that several widely used evaluation datasets exhibit substantial sensitivity to very small targeted deletions. For this reason, we do not pursue step-wise greedy variants here, though they remain an interesting direction for future work.}

\section{Non-Robustness of NBA Rankings}
\label{sec:nba-results}
To assess whether fragility of Bradley–Terry (BT) rankings extends beyond LLM arenas, we applied our method to historical NBA data. The degree of non-robustness in NBA rankings is comparable to that observed in Chatbot Arena: both require removing $<0.05\%$ of matchups to alter the top spot. In both cases, the explanation may trace back to small BT-score gaps at the top of the leaderboard (See \cref{fig:top20-nba}). One plausible explanation is that the skill levels among elite NBA teams are tightly clustered, and so any apparent differences in skill level may not be substantive.

This non-robustness in the NBA datasets suggests a broader conclusion that BT-based evaluation systems--whether in AI benchmarking or sports--tend to be unstable when the margin between competitors is narrow.

\section{Additional Supporting Figures}
\label{sec:added-visualizations}
The figures in this section provide additional insights related to our analysis. Figure~\ref{fig:model_occurrences_cba} shows the distribution of model appearances in Chatbot Arena, respectively, revealing differences in evaluation density and coverage across platforms. Figure~\ref{fig:robustness-vs-score-gap} illustrates the relationship between the robustness of model rankings and the BT score gap between adjacent models, confirming that small score differences tend to coincide with greater sensitivity to worst-case data-dropping.
\label{sec:summary-stats-appendix}
\begin{figure}[ht]
    \centering
    \includegraphics[width=0.65\linewidth]{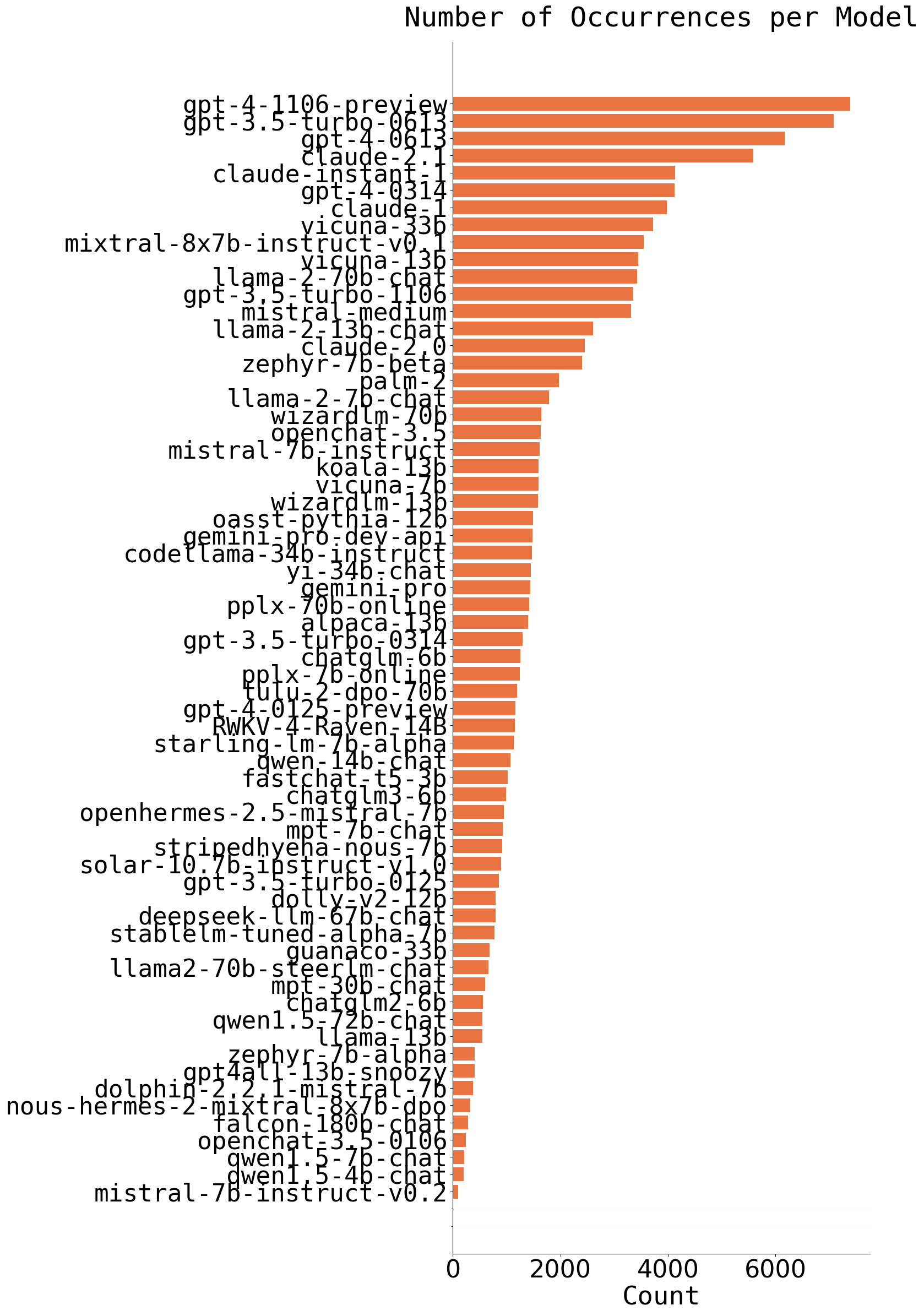}
    \caption{\revision{The number of times each model appears in a match in Chatbot Arena.} The horizontal bar chart shows how frequently each model appeared \revision{in any match}, with GPT-4 and GPT-3.5 variants being the most represented.}
    \label{fig:model_occurrences_cba}
\end{figure}
% \begin{figure}[h]
%     \centering
%     \includegraphics[width=0.8\linewidth]{figures/num_occurrences_mtb.png}
%     \caption{\revision{The number of times each model appears in a match} in MT-Bench. The horizontal bar chart shows how frequently each model appeared \revision{in any match}, with GPT-3.5-turbo and Llama-13b among the most evaluated.}
%     \label{fig:model_occurrences_mtb}
% \end{figure}
% \begin{figure}[t]
%     \centering
%     \includegraphics[width=0.8\linewidth]{figures/top20_CBA.png}
%     \caption{The Top-20 model rankings in Chatbot Arena.}
%     \label{fig:top20-chatbot-arena}
% \end{figure}
\begin{figure}[ht]
    \centering
    \includegraphics[width=0.80\linewidth]{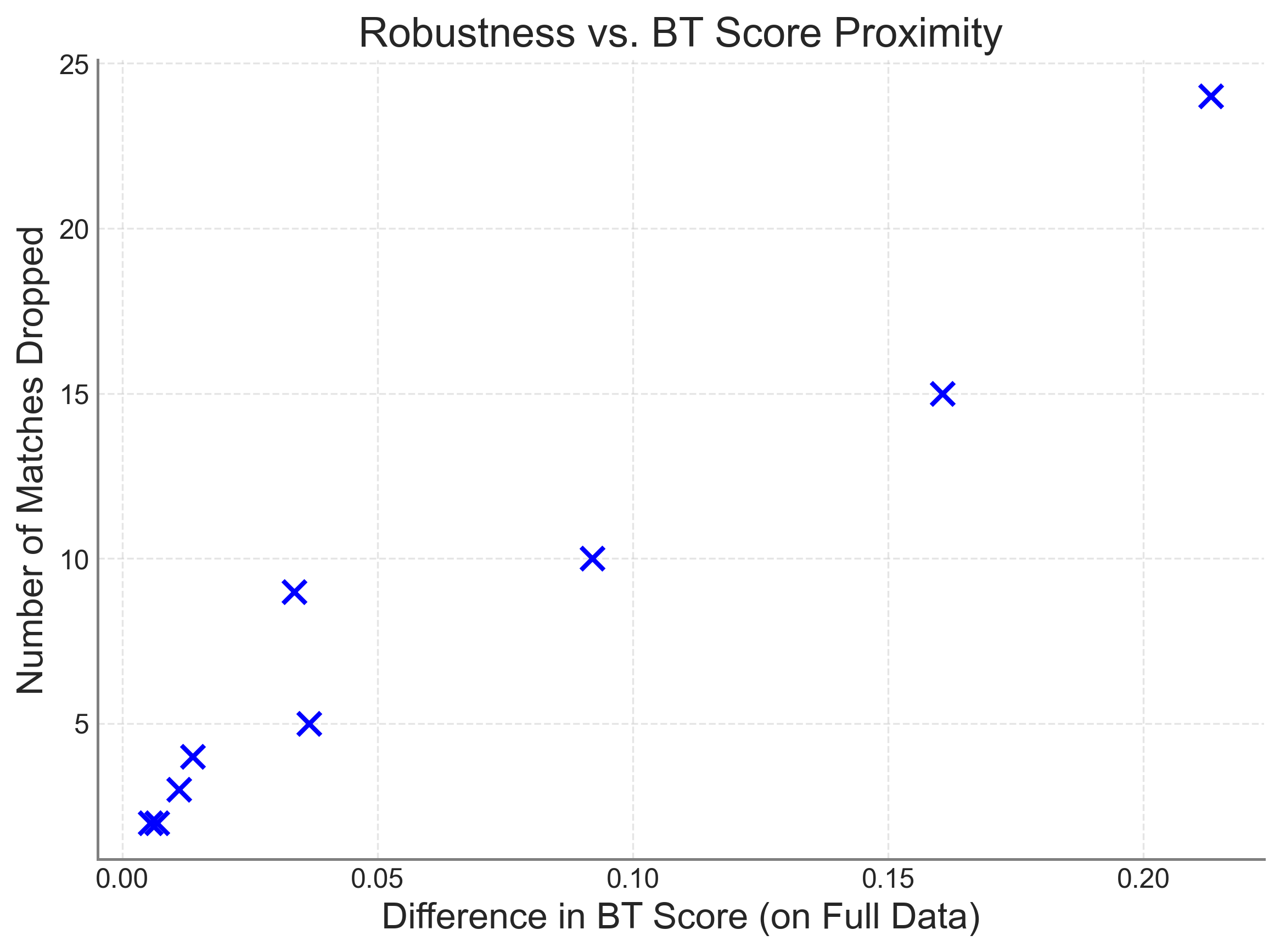}
    \caption{Robustness of results is correlated with the proximity of the BT scores. Each point represents a pair of models whose relative rankings flipped after dropping a small fraction of matchups. In every case, the flip causes one model to enter the top-k rankings (for some $k \in \{1,3,5,10,20\}$) while the other is demoted. These points are taken from both human and LLM-as-a-judge evaluation platforms.}
    \label{fig:robustness-vs-score-gap}
\end{figure}

\section{Large language model (LLM) use}
We used LLMs for grammar checks and to polish writing, to help find sports datasets that yielded the discovery of \citet{Sackmann2024TennisATP}, \rev{and as a judge model in the qualitative study described in \cref{sec:dropped-preferences}.} Although our study is about LLM rankings, we did not use LLMs as direct study subjects.

\end{document}